\documentclass[twoside,11pt]{article}

\usepackage{jmlr2e}
\usepackage{float}
\usepackage{caption}
\usepackage[utf8]{inputenc} % allow utf-8 input
\usepackage[T1]{fontenc}    % use 8-bit T1 fonts
\usepackage{hyperref}       % hyperlinks
\usepackage{url}            % simple URL typesetting
\usepackage{booktabs}       % professional-quality tables
\usepackage{amsfonts}       % blackboard math symbols
\usepackage{nicefrac}       % compact symbols for 1/2, etc.
\usepackage{microtype}      % microtypography
\usepackage{xcolor}         % colors
\usepackage{amsmath,amsfonts, amssymb, amsthm}
\usepackage{algorithm, algpseudocode}
\usepackage{enumitem}
\usepackage{graphicx}
\usepackage{wrapfig}
\usepackage{subfigure}
\usepackage{color}

\usepackage{parskip}

\newcommand{\diff}{\mathrm{d}}
\newcommand{\R}{\mathbb{R}}
\newcommand{\expec}[1]{\mathbb{E}\left[#1\right]}

\newtheorem{theorem}{Theorem}[section]
\newtheorem{lemma}{Lemma}[section]
\newtheorem{proposition}{Proposition}[section]

\theoremstyle{definition}
\newtheorem{definition}{Definition}[section]

\newtheorem{assump}{Assumption}%[section]

\title{Implicit Compressibility of Overparametrized Neural Networks Trained with Heavy-Tailed SGD}

\definecolor{apricot}{rgb}{0.98, 0.81, 0.69}

\usepackage[textsize=scriptsize]{todonotes}

\author{\name{Yijun Wan}$^*$ \email{wan.yijun@huawei.com}\\
 \addr Paris Research Center, Huawei Technologies France
  \AND 
  \name{Melih Barsbey}$^*$ \email{melih.barsbey@boun.edu.tr}\\
 \addr Bo\u{g}aziçi University, \.{I}stanbul, Turkey
  \AND 
  \name{Abdellatif Zaidi} 
\email{abdellatif.zaidi@univ-eiffel.fr}\\
\addr Universit\'e Gustave Eiffel, France and Paris Research Center, Huawei Technologies France
\AND \name{Umut \c{S}im\c{s}ekli} \email{umut.simsekli@inria.fr} \\
 \addr Inria, CNRS, Ecole Normale Sup\'{e}rieure, 
PSL Research University, Paris, France
}

\makeatletter
\let\old@makefntext\@makefntext
\renewcommand{\@makefntext}[1]{%
  \@setpar{\@@par\@tempdima \hsize 
           \advance\@tempdima-15pt\parshape \@ne 15pt \@tempdima}\par
           \parindent 2em\noindent \hbox to \z@{\hss \@thefnmark\hfil}#1%
}
\makeatother

\begin{document}

\renewcommand{\thefootnote}{$^*$}

\footnotetext[1]{Equal contribution.}
\maketitle
\makeatletter
\let\@makefntext\old@makefntext % Revert \@makefntext to original definition
\makeatother
\renewcommand{\thefootnote}{\arabic{footnote}}

\begin{abstract}
Neural network compression has been an increasingly important subject, not only due to its practical relevance, but also due to its theoretical implications, as there is an explicit connection between compressibility and generalization error. Recent studies have shown that the choice of the hyperparameters of stochastic gradient descent (SGD) can have an effect on the compressibility of the learned parameter vector. These results, however, rely on unverifiable assumptions and the resulting theory does not provide a practical guideline due to its implicitness. In this study, we propose a simple modification for SGD, such that the outputs of the algorithm will be provably compressible without making any nontrivial assumptions. We consider a one-hidden-layer neural network trained with SGD, and show that if we inject additive heavy-tailed noise to the iterates at each iteration, for \emph{any} compression rate, there exists a level of overparametrization such that the output of the algorithm will be compressible with high probability. To achieve this result, we make two main technical contributions: (i) we prove a ``propagation of chaos'' result for a class of heavy-tailed stochastic differential equations, and (ii) we derive error estimates for their Euler discretization. Our experiments suggest that the proposed approach not only achieves increased compressibility with various models and datasets, but also leads to robust test performance under pruning, even in more realistic architectures that lie beyond our theoretical setting.
%# with a slight compromise from the training and test error. \umut{need to update with new highlights}
\end{abstract}
\section{Introduction}

Obtaining compressible neural networks has become an increasingly important task in the last decade, and it has essential implications from both practical and theoretical perspectives. From a practical point of view, as the modern network architectures might contain an excessive number of parameters, compression has a crucial role in terms of deployment of such networks in resource-limited environments \citep{oneillOverviewNeuralNetwork2020,blalock2020state}. On the other hand, from a theoretical perspective, several studies have shown that compressible neural networks should achieve a better generalization performance due to their lower-dimensional structure \citep{arora2018stronger,suzuki2018spectral,Suzuki2020Compression,hsu2021generalization,barsbey2021heavy,sefidgaran2022rate}.

Despite their evident benefits, it is still not yet clear how to obtain compressible networks with provable guarantees. In an empirical study, \citet{frankle2018lottery} introduced the ``lottery ticket hypothesis'', which indicated that a randomly initialized neural network will have a sub-network that can achieve a performance that is comparable to the original network; hence, the original network can be compressed to the smaller sub-network. This empirical study has formed a fertile ground for subsequent theoretical research, which showed that such a sub-network can indeed exist (see e.g., \citealp{malach2020proving,burkholz2021existence,da2022proving}). However, it is not clear how to develop an algorithm that can find it in a feasible amount of time.

Another line of research has developed methods to enforce compressibility of neural networks by using sparsity enforcing regularizers (see e.g.,  \citealp{papyan2018theoretical,aytekin2019compressibility,chen2020neural,lederer2023statistical,kengne2023sparse}). While they have led to interesting algorithms, these typically require higher computational resources due to the increased complexity of the problem. On the other hand, due to the nonconvexity of the overall objective, it is also not trivial to provide theoretical guarantees for the compressibility of the resulting network weights. 
% \umut{verify this}.  

Recently it has been shown that the training dynamics can have an influence on the compressibility of the algorithm output. In particular, motivated by the research that produced empirical and theoretical evidence that heavy-tails might arise in stochastic optimization (see e.g., \citealp{martin2019traditional,pmlr-v97-simsekli19a,csimcsekli2019heavy,csimcsekli2020fractional,zhou2020towards,zhang2020adaptive,camuto2021asymmetric}), \citet{barsbey2021heavy} and \citet{shin2021compressing} showed that the network weights learned by stochastic gradient descent (SGD) will be compressible if we assume that they are heavy-tailed and that there exists a certain form of statistical independence within the network weights. These studies illustrated that, even \emph{without} any modification to the optimization algorithm, the learned network weights can be compressible depending on the algorithm hyperparameters (such as the step size, i.e. learning rate, or the batch size). Even though the tail and independence conditions were recently relaxed by \citet{lee2022deep}, the resulting theory relies on unverifiable assumptions, and hence does not provide a practical guideline.

In this paper, we focus on single-hidden-layer neural networks with a fixed second layer (i.e., the setting used in previous work, \citealp{PoC_gaussian}) trained with vanilla SGD, and show that, when the iterates of SGD are simply perturbed by heavy-tailed noise with infinite variance (similar to the settings considered in \citealp{csimcsekli2017fractional,nguyen2019non,csimcsekli2020fractional,huang2021approximation,zhang2023ergodicity}), the assumption made by \citet{barsbey2021heavy} in effect holds. More precisely, denoting the number of hidden units by $n$ and the step size of SGD by $\eta$, we consider the \emph{mean-field limit}, where $n$ goes to infinity and $\eta$ goes to zero. We show that in this limiting case, the columns of the weight matrix will be independent and identically distributed (i.i.d.) with a common \emph{heavy-tailed} distribution.
% More precisely, we represent SGD as a discretization of a stochastic differential equation (SDE), which has both a Gaussian component (corresponding to minibatch noise) and a heavy-tailed component (injected noise).
Then, we focus on the finite $n$ and $\eta$ regime and we prove that for \emph{any} compression ratio (to be precised in the next section), there exists a number $N$, such that if $n \geq N$ and $\eta$ is sufficiently small, the network weight matrix will be compressible with high probability. Figure~\ref{plt:illus} illustrates the overall approach and precises our notion of compressibility. 

\begin{figure}
\centering
% \vspace{-0.65cm}
\includegraphics[width=0.65\textwidth]{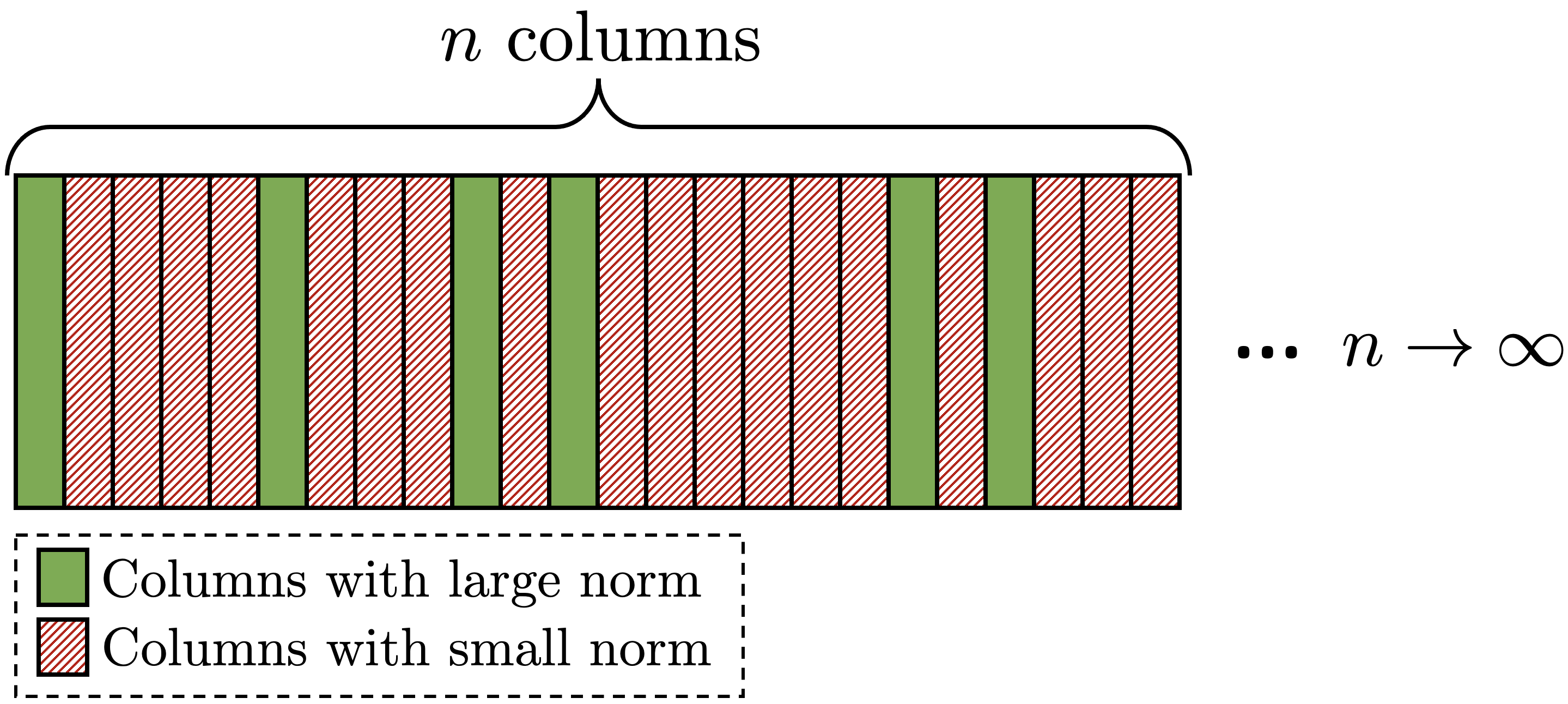}
% \vspace{-0.5cm}
  \caption{The illustration of the overall approach. We consider a one-hidden-layer neural network with $n$ hidden units, which results in a weight matrix of $n$ columns (first layer). We show that, when SGD is perturbed with heavy-tailed noise, as $n\to \infty$, each column will follow a multivariate heavy-tailed distribution in an i.i.d.\ fashion. This implies that a small number of columns will have significantly larger norms compared to the others; hence, the norm of the overall weight matrix will be determined by such columns \citep{compressible_distribution}. As a result, the majority can be removed (i.e., set to zero), which we refer to as compressibility.}
 \vspace{-0.15cm}
\label{plt:illus}
\end{figure}

To prove our compressibility result, we make two main technical contributions. We first consider the case where the step size $\eta \to 0$, for which the SGD recursion perturbed with heavy-tailed noise yields a \emph{system} of heavy-tailed stochastic differential equations (SDE) with $n$ particles. As our first technical contribution, we show that as $n \to \infty$ this particle system converges to a mean-field limit, which is a McKean-Vlasov-type SDE that is driven by a heavy-tailed process \citep{jourdain2007nonlinear,liang2021exponential,Cavallazzi}. For this convergence, we obtain a rate of $n^{-1/2}$, which is faster than the best known rates, as recently proven by \citet{Cavallazzi}. This result indicates that a \emph{propagation of chaos} phenomenon \citep{sznitman1991topics} emerges\footnote{Here, the term chaos refers to statistical independence: when the particles are initialized independently, they stay independent through the whole process even though their common distribution might evolve.}: in the mean-field regime, the columns of the weight matrix will be i.i.d.\ and heavy-tailed due to the injected noise.

Next, we focus on the Euler discretizations of the particle SDE to be able to obtain a practical, implementable algorithm. As our second main technical contribution, we derive \emph{strong-error} estimates for the Euler discretization \citep{kloeden1992stochastic} and show that for sufficiently small $\eta$, the trajectories of the discretized process will be close to the one of the continuous-time SDE, in a precise sense. This result is similar to the ones derived for vanilla SDEs (e.g., \citealp{mikulevivcius2018rate}) and enables us to incorporate the error induced by using a finite step size $\eta$ to the error of the overall procedure.

Equipped with these results, we finally prove a high-probability compression bound by invoking \citep{compressible_distribution,amini2011compressibility}, which essentially shows that an i.i.d.\ sequence of heavy-tailed random variables will have a small proportion of elements that will dominate the whole sequence in terms of absolute values (to be stated formally in the next section). This establishes our main contribution.
Here, we shall note that similar mean-field regimes have already been considered in machine learning (see e.g., \citealp{mei2018mean,chizat2018global,rotskoff2018trainability,jabir2019mean,mei2019mean,PoC_gaussian,sirignano2022mean}). However, these studies all focused on particle SDE systems that either converge to deterministic systems or that are driven by Brownian motion. While they have introduced interesting analysis tools, we cannot directly benefit from their analysis in this paper, since the heavy-tails are crucial for obtaining compressibility, and the Brownian-driven SDEs cannot produce heavy-tailed solutions in general. Hence, as we consider heavy-tailed SDEs in this paper, we need to use different techniques to prove mean-field limits, compared to the prior art in machine learning.  

To validate our theory, we conduct experiments with various neural networks and datasets. Our results show that, even with a minor modification to SGD (i.e., injecting heavy-tailed noise), the proposed approach
can achieve compressibility with a negligible computational overhead and with a slight compromise from the training and test error. Our findings further demonstrate that our methodology generalizes beyond our theoretical results, and produces models that are not only compressible, but robust in terms of test performance in fully connected neural networks with single or multiple hidden layers, and convolutional neural networks, implying that our approach is indeed a promising one in terms of its practical implications.

\section{Preliminaries and Technical Background}

\textbf{Notation.} For a vector $u\in\R^d$, denote by $\|u\|$ its Euclidean norm, and by $\|u\|_p$ its $\ell_p$ norm. For a function $f\in C(\R^{d_1}, \R^{d_2})$, denote by $\|f\|_\infty:=\sup_{x\in\R^{d_1}}\|f(x)\|$ its $L^\infty$ norm. For a family of $n$ (or infinity) vectors, the indexing $\cdot^{i,n}$ denotes the $i$-th vector in the family. In addition, for random variables, $\overset{(\mathrm{d})}{=}$ means equality in distribution, and the space of probability measures on $\R^d$ is denoted by $\mathcal{P}(\R^d)$. For a matrix $A\in\R^{d_1\times d_2}$, its Frobenius norm is denoted by $\|A\|_F = \sqrt{\sum_{i=1}^{d_1}\sum_{j=1}^{d_2}|a_{i,j}|^2}$. Unless otherwise noted, $\mathbb{E}$ denotes the expectation over all the randomness taken into consideration. 

% \subsection{Alpha-stable processes}
\textbf{Alpha-stable processes.} 
A centered random variable $X \in \mathbb{R}^d$ is called \emph{$\alpha$-stable} with the stability parameter $\alpha \in (0,2]$, if $X_1,\, X_2,\,\ldots$ are independent copies of $X$, then
   $n^{-1/\alpha}\sum_{j=1}^n X_j \overset{(\mathrm{d})}{=} X \text{ for all } n\ge 1$ \citep{samoradnitsky2017stable}.
Stable distributions appear as the limiting distribution in the generalized
central limit theorem (CLT) \citep{gnedenko_kolmogorov}. In the one-dimensional case ($d=1$), we call the variable $X$ a symmetric $\alpha$-stable random variable if its characteristic function is of the following form: $\mathbb{E}[\exp(i\omega X)]=\exp(-|\lambda\omega|^\alpha)$ for $\omega\in \mathbb{R}$ and some $\lambda\in\mathbb{R}_+$.

For symmetric $\alpha$-stable distributions, the case $\alpha=2$ corresponds to the Gaussian distribution, while $\alpha=1$ corresponds to the Cauchy distribution. An important property of $\alpha$-stable distributions is that in the case $\alpha\in(1,2)$, the $p$-th moment of an $\alpha$-stable random variable is finite if and only if $p<\alpha$; hence, the distribution is heavy-tailed. In particular, $\mathbb{E}[|X|]< \infty$ and $\mathbb{E}[|X|^2] = \infty$, which can be used to model phenomena with heavy-tailed observations. 

In this paper, as perturbations to be added to the iterates, we consider the three most common types of $\alpha$-stable random vectors that have been used in finance \citep{mandelbrot63, cont01}, statistical physics \citep{Montroll84}, and engineering literature \citep{nikiasBook}. We first describe these random vectors, and then provide some intuition regarding their behavior.

\begin{itemize}[topsep=0pt,leftmargin=.11in]
    \item \textbf{Type-I.} Let $Z \in \mathbb{R}$ be a symmetric $\alpha$-stable random variable. We then construct the random vector $X$ such that all the coordinates of $X$ is equated to $Z$. In other words $X = \mathbf{1}_d Z$, where $\mathbf{1}_d \in \mathbb{R}^d$ is a vector of ones. With this choice, $X$ admits the following characteristic function:
    $
    \expec{\exp(i\langle u,X\rangle}=\exp(-|\langle u, \mathbf{1}_d \rangle|^\alpha)
    $;
    \item \textbf{Type-II.} $X$ has i.i.d.\ coordinates, such that each component of $X$ is a symmetric $\alpha$-stable random variable in $\mathbb{R}$. This choice yields the following characteristic function: $
    \expec{\exp(i\langle u,X\rangle}=\exp(-\sum_{i=1}^d|u_i|^\alpha)$;

    \item \textbf{Type-III.} $X$ is rotationally invariant $\alpha$-stable random vector with the characteristic function $\expec{\exp(i\langle u,X\rangle}=\exp(-\|u\|^\alpha)$. 
\end{itemize}

Notice that when added to a parameter vector (e.g. corresponding to a neuron), Type-I noise disturbs all parameters in the same direction and magnitude, e.g. acting like a random bias node scaled by the input. In contrast, Type-II noise constitutes an i.i.d. perturbation that affects each parameter separately, allowing some of the noise components to be very large while others are small, and/or in opposite directions. Lastly, due to the heavy-tailed distribution of its norm, Type-III noise vectors are likely to include elements that are simultaneously large or small in magnitude, yet these elements can vary among themselves in magnitude and direction. Also note that the Type-II and Type-III noises reduce to a Gaussian distribution when $\alpha =2$, i.e., the characteristic function becomes $\exp(-\|\lambda u\|^2)$. 
% characteristic function of $\alpha$-stable random variable $X$ is given by (for some $S\alpha S$?)
% \[\expec{\exp(itX)} = ...\]
% where $\alpha \in (0, 2]$ is the tail-index and $\sigma \in (0, \infty)$ is the scale parameter. , componentwise independence

% If $X$ is symmetric, then there exists a scale parameter $c>0$ and the characteristic function of $X$ is given by %$\expec{\exp(i \scp{\xi}{X})} = \exp(-|c\xi|^\alpha)$ for all $\xi\in\R^d$. If $X$ is a random vector with independent components, then there exists a scale vector~$\mathbf{c}=(c^1,\ldots,c^d) \in \mathbb{R}^d_+$ and the characteristic function of $X$ is given by~$\E{\exp(i \scp{\xi}{X})} = \exp(-\sum_{j=1}^d|c^j\xi^j|^\alpha)$.

% \begin{itemize}[topsep=0pt,leftmargin=.11in]
% \item $\Lm_0=0$ almost surely;
% \item For any $t_{0}<t_{1}<\cdots<t_{N}$, the increments $\Lm_{t_{n}}-\Lm_{t_{n-1}}$
% are independent;
% \item The difference $\Lm_{t}-\Lm_{s}$ and $\Lm_{t-s}$
% have the same distribution, with the characteristic function $\exp(- (t-s)^\alpha\|u\|_2^\alpha)$ for $t>s$;
% \item $\Lm_{t}$ has stochastically continuous sample paths, i.e.
% for any $\delta>0$ and $s\geq 0$, $\mathbb{P}(\|\Lm_{t}-\Lm_{s}\|>\delta)\rightarrow 0$
% as $t\rightarrow s$.
% \end{itemize}

Similar to the fact that stable distributions extend the Gaussian distribution, we can define a more general random process, called the \emph{$\alpha$-stable L\'evy process}, that extends the Brownian motion. 
Formally, $\alpha$-stable processes are stochastic processes $(\mathrm{L}^\alpha_t)_{t \geq 0}$ with independent and stationary $\alpha$-stable increments, and have the following definition:
\begin{itemize}[topsep=0pt,leftmargin=.11in]
\item $\mathrm{L}^\alpha_0=0$ almost surely,
\item For any $0\le t_{0}<t_{1}<\cdots<t_{N}$, the increments $\mathrm{L}^\alpha_{t_{n}}-\mathrm{L}^\alpha_{t_{n-1}}$
are independent,
\item For any $0\le s< t$, the difference $\mathrm{L}^\alpha_{t}-\mathrm{L}^\alpha_{s}$ and $(t-s)^{1/\alpha}\mathrm{L}^\alpha_1$
have the same distribution,
\item $\mathrm{L}^\alpha_{t}$ is stochastically continuous, i.e.
for any $\delta>0$ and $s\geq 0$, $\mathbb{P}(\|\mathrm{L}^\alpha_{t}-\mathrm{L}^\alpha_{s}\|>\delta)\rightarrow 0$
as $t\rightarrow s$.
\end{itemize}
To fully characterize an $\alpha$-stable process, we further need to specify the distribution of $\mathrm{L}^\alpha_{1}$. Along with the above properties, the choice for $\mathrm{L}^\alpha_{1}$ will fully determine the process. For this purpose, we will again consider the previous three types of $\alpha$-stable vectors: We will call the process $\mathrm{L}^\alpha_{t}$ a Type-I process if $\mathrm{L}^\alpha_{1}$ is a Type-I $\alpha$-stable random vector. We define the Type-II and Type-III processes analogously. Note that, when $\alpha =2$, Type-II and Type-III processes reduce to the Brownian motion. For notational clarity, occasionally, we will drop the index $\alpha$ and denote the process by $\mathrm{L}_{t}$.

% meaning that for any $u,t\ge 0$,
% \[
% \mathrm{L}^\alpha_{u+t} - \mathrm{L}^\alpha_u \overset{\text{(d)}}{=} t^{1/\alpha}\mathrm{L}^\alpha_1\quad \text{is independent of} \quad (\mathrm{L}^\alpha_t)_{0\le t\le u}.
% \]

% \subsection{Compressibility of heavy-tailed processes}
\textbf{Compressibility of heavy-tailed processes. }
One interesting property of heavy-tailed distributions in the one-dimensional case is that they exhibit a certain compressibility property. Informally, if we consider a sequence of i.i.d.\ random variables coming from a heavy-tailed distribution, a small portion of these variables will likely have a very large magnitude due to the heaviness of the tails, and they will dominate all the other variables in magnitude \citep{nair2022fundamentals}. Therefore, if we only keep this small number of variables with large magnitude, we can ``compress'' (in a lossy way) the whole sequence of random variables by representing it with this small subset.  

% 
% The notion of compressibility for infinite i.i.d.\ sequences based on the asymptotic behavior of the truncated ordered statistics is introduced in \citep{compressibility}, which depends on the
% decay of the probability density function. In particular, they demonstrated the connection between their notion of compressibility (also called s-compressibility) and heavy-tail (polynomially decaying) distributions, for example, $\alpha$-stable distributions if $\alpha<2$. 

Concurrently, \citet{amini2011compressibility,compressible_distribution} provided formal proofs for these explanations. Formally,   \citet{compressible_distribution} characterized the family of probability distributions whose i.i.d.\ realizations are compressible. They introduced the notion of $\ell_p$-compressibility - in terms of the error made after pruning a fixed portion of small (in magnitude) elements of an i.i.d.\ sequence, whose common distribution has diverging $p$-th order moments. More precisely, let $X_n=(x_1, \ldots, x_n)$ be a sequence of i.i.d.\ random variables such that $\expec{|x_1|^\alpha}=\infty$ for some $\alpha\in \mathbb{R}_+$. Then, for all $p\ge\alpha$ and $0<\kappa\le 1$ denoting by $X_n^{(\kappa n)}$ the $\lfloor\kappa n\rfloor$ largest ordered statistics\footnote{In other words, $X_n^{(\kappa n)}$ is obtained by keeping only the largest (in magnitude) $\kappa n$ elements of $X_n$ and setting all the other elements to $0$.} of $X_n$, the following asymptotic on the relative compression error holds almost surely:
\begin{align*}
\lim_{n\to\infty} \frac{\|X_n^{(\kappa n)} -X_n\|_p} { \|X_n\|_p} = 0    
\end{align*}
Built upon this fact, \citet{barsbey2021heavy} proposed structural pruning of neural networks (the procedure described in Figure~\ref{plt:illus}) by assuming that the network weights provided by SGD will be asymptotically independent. In this study, instead of making this assumption, we will directly prove that the network weights will be asymptotically independent in the two layer (i.e. single-hidden-layer) neural network setting with additive heavy-tailed noise injections to SGD. 
% for a large number of neurons, which enables the storage and computation reduction with guaranteed generalizability while still maintaining the accuracy of large neural networks. 

\section{Problem Setting and the Main Result}
\label{sec:problem-setting}
% \subsection{Problem Setup}
% \label{sec:setup}
% \begin{itemize}
%     \item Consider a Lipchitz non-linear function $\sigma: \R \to \R$.  A simple neural network with one hidden layer of $n$ neurons can be written as
% \[
% f_n(x,\Theta^n) = \frac{1}{n}\sum_{i=1}^n \sigma(X^T \theta^{i,n}).
% \]

% (Rewrite) The risk $R(\rho_\theta)$ is convex in $\rho_\theta$, and its dependence on $\theta$ is denoted by $L(\theta)$, while $L(\theta)$ is not necessarily convex in $\theta$ due to the nonlinearity in $\rho$. The goal is to minimize over $\theta\in\Theta$ the loss $L(\theta)$ using gradient descent
% \[
% \theta^{k+1} = \theta^k - \eta\partial_\theta L(\theta^k).
% \]
% nevertheless in practice, learning algorithms only have access to the empirical risk based on an incomplete data set, giving rise to statistical errors. It is not clear that how the minimizor on the empirical risk will perform on the population risk, in other words, the generalization error.

% \end{itemize}

We consider a single-hidden-layer overparametrized network of $n$ units and use the setup provided in \citep{PoC_gaussian}. 
Our goal is to minimize the expected loss in a supervised learning regime, where for each data $z=(x,y)$ distributed according to $\pi(\diff x,\diff y)$,\footnote{Note that for finite datasets, $\pi$ can be chosen as a measure supported on finitely many points.} the feature $x$ is included in $\mathcal{X}\subset \R^d$ and the label $y$ is in $\mathcal{Y}$. We denote by $\theta^{i,n}\in \R^p$ the parameter for the $i$-th unit, and the parametrized model is denoted by $h_{x}: \mathbb{R}^p \to \mathbb{R}^l$. The mean-field network  is the average over models for $n$ units:
\[
f_{\Theta^n}(x) = (1/n)\sum\nolimits_{i=1}^n h_x (\theta^{i,n}),
\]
where $\Theta^n=(\theta^{i,n})_{i=1}^n \in \mathbb{R}^{p \times n} $ denotes the collection of parameters in the network  and $x\in\mathcal{X}$ is the feature variable for the data point.  In particular, the mean-field network corresponds to a two-layer neural network with the weights of the
second layer are fixed to be $1/n$ and $\Theta^n$ is the parameters of the first layer. While this model is less realistic than the models used in practice, we believe that it is desirable from theoretical point of view, and this defect can be circumvented upon replacing $h_x(\theta^{i,n})$ by $h_x(c^{i,n},\theta^{i,n}) = c^{i,n}h_x(\theta^{i,n})$, where $c^{i,n}$ and $\theta^{i,n}$ are
weights corresponding to different layers. However, in order to obtain similar results in this setup as in our paper, stronger assumptions are inevitable and the proof should be more involved, which are left for future work.

Given a loss function $\ell: \mathbb{R}^l\times \mathcal{Y} \to \mathbb{R}^+$, the goal (for each $n$) is to minimize the expected loss taken over the distribution over the whole dataset $\pi$,
\begin{equation}\label{eq:loss}
R(\Theta^n) = \mathbb{E}_{(x,y)\sim \pi} \left[\ell\left( f_{\Theta^n}(x),y\right) \right] .
\end{equation}
One of the most popular approaches to minimize this loss is the stochastic gradient descent (SGD) algorithm. In this study, we consider a simple modification of SGD, where we inject a stable noise vector to the iterates at each iteration. For notational clarity, we will describe the algorithm and develop the theory over gradient descent, where we will assume that the algorithm has access to the true gradient $\nabla R$ at every iteration. However, since we are already injecting a heavy-tailed noise with \emph{infinite variance}, our techniques can be adapted for handling the stochastic gradient noise (under additional assumptions, e.g., \citealp{PoC_gaussian}), which typically has a milder behavior compared to the $\alpha$-stable noise\footnote{In \citealp{pmlr-v97-simsekli19a} the authors argued that the stochastic gradient noise in neural networks can be modeled by using stable distributions. Under such an assumption, the effect of the stochastic gradients can be directly incorporated into $\mathrm{L}_t^\alpha$. }. 
% with stable noise injection to improve compressibility of the overparametrized network when $n$ is sufficiently large.

Let us set the notation for the proposed algorithm. Let $\hat\theta^{i,n}_0$, $i=1,\ldots,n$, be the initial values of the iterates, which are $n$ random variables in $\R^d$ distributed independently according to a given initial probability distribution $\mu_0$. Then, we consider the gradient descent updates with stepsize  $\eta  n$, which is perturbed by i.i.d.\ $\alpha$-stable noises
 $\sigma\cdot \eta^{1/\alpha}X^{i,n}_k$ for each unit $i=1,\ldots, n$, $\alpha\in (1,2)$ and some $\sigma>0$: 
 % Suppose that the dataset is large enough so that we have access to the gradient of the population loss. Then the stochastic gradient descent dynamics is given by
% \begin{equation}\label{eq:SGD} 
%     \begin{cases}  
%    \hat\theta^{i,n}_{k+1}   =  \hat\theta^{i,n}_k -\eta n\left[\partial_{\theta^{i,n}}R(\Theta^n_k) \right]+   \sigma\cdot \eta^{1/\alpha} X^{i,n}_k \\
%      \hat\theta^{i,n}_0 \sim \mu_0 \in \mathcal{P}(\R^d),
%      \end{cases}
% \end{equation}
\begin{equation}\label{eq:SGD} 
    % \begin{cases}  
   \hat\theta^{i,n}_{k+1}   =  \hat\theta^{i,n}_k -\eta n\left[\partial_{\theta^{i,n}}R(\Theta^n_k) \right]+   \sigma\cdot \eta^{1/\alpha} X^{i,n}_k 
     % \end{cases}
\end{equation}
where the scaling factor $\eta^{1/\alpha}$ in front of the stable noise enables the discrete dynamics of the system homogenize to SDEs as $\eta\to 0$. Here $\sigma$ is fixed to be a constant. In practice, we tune the stepsize $\eta$ according to the number of neurons $n$, hence influencing the noise level. At this stage, we do not have to determine which type of stable noise (e.g., Type-I, II, or III) that we shall consider as they will all satisfy the requirements of our theory. However, our empirical findings will illustrate that the choice will affect the overall performance.

We now state the assumptions that will imply our theoretical results. The following assumptions are rewritings with a certain degree of relaxation (in terms of the order of moments) of \citep[Assumption A1]{PoC_gaussian}.
\begin{assump}\label{assump:coefficients_regularity}
\begin{itemize}
    \item Regularity of the model: for each $x\in\mathcal{X}$, the function $h_x: \R^p \to \R^l$ is two-times differentiable, and there exists a function $\Psi: \mathcal{X} \to \R_+$ such that for any $x\in\mathcal{X}$,
    \[
    \|h_x(\cdot)\|_\infty + \|\nabla h_x(\cdot)\|_\infty + \|\nabla^2 h_x(\cdot)\|\infty \le \Psi(x).
    \]
    \item Regularity of the loss function: there exists a function $\Phi: \mathcal{Y} \to \R_+$ such that 
    \[   
    \|\partial_1 \ell(\cdot,y)\|_\infty + \|\partial^2_1 \ell(\cdot,y)\|_\infty \le \Phi(y)
    \]
    \item Moment bounds on $\Phi(\cdot)$ and $\Psi(\cdot)$: there exists a positive constant $B$ such that
    \[
    \mathbb{E}_{(x,y)\sim \pi}[\Psi^2(x)(1+\Phi^2(y))] \le B^2.
    \]
\end{itemize}
\end{assump}
Let us remark that Assumption \ref{assump:coefficients_regularity} includes the smoothness and boundedness assumptions that have been made in the mean field literature \citep{mei2018mean,mei2019mean} and are satisfied by several smooth activation functions, including the sigmoid and hyper-tangent functions.
% \begin{itemize}
%     \item Standard smoothness
% assumptions in the mean field literature, see Mei et al., 2018, 2019.
%     \item .
% \end{itemize}

We now proceed to our main result. Let $\hat\Theta^n_k \in \mathbb{R}^{p \times n}$ be the matrix with columns being the parameters $\hat\theta^{i,n}_k$, $i=1,\ldots,n$ obtained by the recursion \eqref{eq:SGD} after $k$ iterations. We will now compress $\hat\Theta^n_k$ by pruning its columns with small norms. More precisely, fix a compression ratio $\kappa\in (0,1)$, compute the norms of the columns of $\hat\Theta^n_k$, i.e.,   $\|\hat\theta^{i,n}_k\|$. Then, keep the $\lfloor \kappa n\rfloor$ columns, which have the largest norms, and set all the other columns to zero in entirety. Finally, denote by $\hat\Theta^{(\kappa n)}_k \in \mathbb{R}^{p \times n}$, the pruned version of $\hat\Theta^n_k$.  

% For any , denote by  the concatenation of all parameters after pruning the $\lfloor\kappa n\rfloor$ smallest $\|\hat\theta^{i,n}_t\|$'s (setting them to be zero if their norms are among the $\lfloor \kappa n\rfloor$ smallest norms of $\|\hat\theta^{i,n}_k\|$). 

\begin{theorem}
\label{thm:main}
Suppose that Assumption~\ref{assump:coefficients_regularity} holds. For any fixed $t>0$, $\kappa \in (0,1)$ and $\epsilon>0$ sufficiently small, with probability $1-\epsilon$, there exists $N\in\mathbb{N}_+$ such that for all $n\ge N$ and $
\eta$ such that $\eta \le n^{-\alpha/2-1}$, the following upper bound on the relative compression error for the parameters holds:
% \[
% \frac{\|\sigma_{\lfloor\kappa n\rfloor}(\mathbf{w}^n)\|}{\|\mathbf{w}^n\|} \le \epsilon.
% \]
\[
\frac{\left\|\hat\Theta^{(\kappa n)}_{\lfloor t/\eta \rfloor} - \hat\Theta^n_{\lfloor t/\eta \rfloor} \right\|_F}{\left\|\hat\Theta^n_{\lfloor t/\eta \rfloor} \right\|_F} \le \epsilon .
\]
\end{theorem}
This bound shows that, thanks to the heavy-tailed noise injections, the weight matrices will be compressible at \emph{any} compression rate, as long as the network is sufficiently overparametrized and the step size is sufficiently small. We shall note that this bound also enables us to directly obtain a generalization bound by invoking~\citep[Theorem 4]{barsbey2021heavy}. 

% \umut{mention incompressibility and local convergence}

% \end{remark}

\section{Proof Strategy and Intermediate Results}

In this section, we gather the main technical contributions with the purpose of demonstrating Theorem~\ref{thm:main}. We begin by rewriting \eqref{eq:SGD} in the following form:
\begin{equation}\label{eq:SGD2} 
    % \begin{cases}  
   \hat\theta^{i,n}_{k+1} - \hat\theta^{i,n}_k  =   \eta b(\hat\theta^{i,n}_k, \hat\mu^n_k) +   \sigma\cdot \eta^{1/\alpha} X^{i,n}_k %\\
     % \hat\theta^{i,n}_0 \sim \mu_0 \in \mathcal{P}(\R^d),
     % \end{cases}
\end{equation}
where $\hat\mu^n_k = \frac{1}{n} \sum_{i=1}^n \delta_{\hat\theta^{i,n}_k}$ is the empirical distribution of parameters at iteration $k$ and $\delta$ is the Dirac measure, and the drift is given by
% $b(\theta^{i,n}_k,\mu^n_k) = n\partial_{\theta^{i,n}}R(\Theta^n)$ is the gradient of the risk. Here the function $b$ is deliberately parametrized by $\theta^{i,n}_k$ and $\hat\mu^n_k$, due to the following identity:
$b(\theta^{i,n}_k,\mu^n_k) = - \mathbb{E}[\partial_1 \ell(\mu^n_k(h_x(\cdot)),y)\nabla h_x(\theta^{i,n}_k)]$,
where $\partial_1$ denotes the partial derivative with respect to the first parameter and 
%\umut{fix the next line, too long}
% $\mu^n = \frac{1}{n}\sum_{i=1}^n \delta_{\theta^{i,n}_k}$, 
% and
\[
\mu^n_k(h_x(\cdot)) :=  \frac{1}{n} \sum_{i=1}^n h_x(\theta^{i,n}_k) = f_{\Theta^n_k}(x).
\]
It is easy to check that $b(\theta^{i,n}_k,\mu^n_k) = -n\partial_{\theta^{i,n}}R(\Theta^n)$.
By looking at the dynamics from this perspective, we can treat the evolution of the parameters as a system of evolving probability distributions $\mu^n_k$: the empirical distribution of the parameters during the training process will converge to a limit  as $\eta$ goes to $0$ and $n$ goes to infinity. 

We start by linking the recursion \eqref{eq:SGD} to its limiting case where $\eta \to 0$. The limiting dynamics can be described by the following system of SDEs: 
\begin{equation}
\label{eq:SGF}
    % \begin{cases}  
   \diff \theta^{i,n}_t =   b(\theta^{i,n}_t,\mu^n_t) \diff t + \sigma\diff \mathrm{L}^{i,n}_t %\\
    % \theta^{i,n}_0 \sim \mu_0 \in \mathcal{P}(\R^d),
     % \end{cases}
\end{equation}
where $\mu^n_t = \frac{1}{n} \sum_{i=1}^{n}\delta_{\theta^{i,n}_t}$ and $(\mathrm{L}^{i,n}_t)_{t\ge 0}$ are independent $\alpha$-stable processes such that $\mathrm{L}^{i,n}_1 \overset{(\mathrm{d})}{=} X^{i,n}_1$. 
% This is a corollary of Theorem \ref{thm:euler}. 
% 
We can now see the original recursion \eqref{eq:SGD} as an Euler discretization of \eqref{eq:SGF} and then we have the following strong uniform error estimate for the discretization. 
\begin{theorem}
\label{thm:euler}

Let $(\theta^{i,n}_t)_{t\ge 0}$ be the solutions to SDE \eqref{eq:SGF} and $(\hat\theta^{i,n}_k)_{k\in\mathbb{N}_+}$ be given by SGD \eqref{eq:SGD} with the same initial condition $\theta^{i,n}_0$ and $\alpha$-stable L\'evy noise $\mathrm{L}^{i,n}_\cdot$, i=1,\ldots,n. Under Assumption \ref{assump:coefficients_regularity}, for any $T>0$, if $\eta k\le T$, there exists a constant $C$ depending on $B,T,\alpha$ such that
\[
\expec{\sup_{i\le n} \|\theta^{i,n}_{\eta k}-\hat\theta^{i,n}_k\|} \le C (\eta n)^{1/\alpha}.
\]
\end{theorem}
In comparison to the standard error estimates in the Euler-Maruyama scheme concerning only the stepsize $\eta$, the additional $n$-dependence is because here we consider the supremum of the approximation error over all $i\le n$, which involves the expectation of the supremum of the modulus of $n$ independent $\alpha$-stable random variables.

Next, we start from the system \eqref{eq:SGF} and consider the case where $n\to \infty$. In this limit, we obtain the following 
% Furthermore, if the number of neurons $n$ is sent to infinity, the $\theta^{i,n}_t$s are asymptotically independent (see Theorem \ref{thm:poc}), evolving according to the 
McKean-Vlasov-type stochastic differential equation: 
\begin{equation}\label{eq:McKean–Vlasov}
    % \begin{cases}  
    \diff \theta^\infty_t = b(\theta^\infty_t,[\theta^\infty_t]) \diff t + \diff \mathrm{L}_t % \\
     % [\theta^\infty_0]=\mu\in \mathcal{P}(\R^d),
     % \end{cases}
\end{equation}
where $(\mathrm{L}_t)_{t\ge 0}$ is an $\alpha$-stable process and $[\theta^\infty_t]$ denotes the distribution of $\theta^\infty_t$. The existence and uniqueness of a strong solution to \eqref{eq:McKean–Vlasov} are given by \citet{Cavallazzi}. Moreover, for any positive $T$, $\expec{\sup_{t\le T}\|\theta^\infty_t\|^\alpha}<+\infty.$ This SDE with measure-dependent coefficients turns out to be a useful mechanism for analyzing the behavior of neural networks and provides insights into the effects of noise on the learning dynamics.

In this step, we will link the system \eqref{eq:SGF} to its limit \eqref{eq:McKean–Vlasov}, which is 
% An important ingredient in the proof towards compressibility is the 
a strong uniform propagation of chaos result for the weights. 
% $\theta^{i,n}_t$, $i=1, \ldots, n$, of $n$ neurons during the stochastic gradient flow with independent injected $\alpha$-stable L\'evy noise. 
The next result shows that, when $n$ is sufficiently large, the trajectories of weights asymptotically behave as i.i.d.\ solutions to \eqref{eq:McKean–Vlasov}.
% the McKean-Vlasov equation associated to common sequence of $\alpha$-stable diffusion terms.   

\begin{theorem}
\label{thm:poc}
Following the existence and uniqueness of strong solutions to \eqref{eq:SGF} and \eqref{eq:McKean–Vlasov}, let $
(\theta^{i,\infty}_t)_{t\ge 0}$ be solutions to the McKean-Vlasov equation \eqref{eq:McKean–Vlasov} and $(\theta^{i,n}_t)_{t\ge 0}$ be solutions to \eqref{eq:SGF} associated with same realization of $\alpha$-stable processes $(\mathrm{L}^i_t)_{t\ge 0}$ for each $i$. Suppose that $(\mathrm{L}^i_t)_{t\ge 0}$ are independent. Then there exists $C$ depending on $T, B$ such that
\[
\expec{\sup_{t\le T}\sup_{i\le n}|\theta^{i,n}_t - \theta^{i,\infty}_t|} \le   \frac{C}{\sqrt{n}} 
\]
\end{theorem}
Our result differs from the existing literature by taking the supremum over the indices $i$ before taking the expectation, which is obviously stronger than taking the supremum over $i$ outside the expectation. It is also worth mentioning that the $O(n^{-1/2})$ decreasing rate here is better, if $\alpha<2$, than the state of the art \citep{Cavallazzi} with classical Lipschitz assumptions on the coefficients of SDEs. The reason is that here, thanks to Assumption \ref{assump:coefficients_regularity}, we can benefit from the one-hidden-layer neural network structure.

% In this section, we gather the main technical contributions with the purpose of demonstrating Theorem \ref{thm:main}: propagation of chaos for a large system of heavy-tailed SDEs and the strong-error estimates for the Euler discretization of SDEs using SGD. We recall our setup in Section \ref{sec:setup}. Let us consider a supervised learning problem with objective function for $\Theta^n=(\theta^{1,n}, \ldots, \theta^{n,n})$
% \[
% R(\Theta^n) = \mathbb{E}_{(x,y)\sim \pi} \left[l\left( f_{\Theta^n}(x),y\right) \right],
% \]
% where the hypothesis function $f_{\Theta^n}(x) = \frac{1}{n}\sum_{i=1}^n h_x (\theta^{i,n})$ can be represented by a single-hidden-layer neural network with the weights of the second layler fixed to be $1/n$ for some non-linearity $h$. 

% When $n\to\infty$, \eqref{eq:SGF} converges to the following McKean-Vlasov equation where the coefficient depends on the distribution of the solution itself:
% \begin{equation}
%     \begin{cases}  
%     \diff \theta^\infty_t = b(\theta^\infty_t,[\theta^\infty_t]) \diff t + \diff \mathrm{L}_t \\
%      [\theta^\infty_0]=\mu\in \mathcal{P}(\R^d)
%      \end{cases},
% \end{equation}
% for which the existence and uniqueness of a strong solution to \eqref{eq:McKean–Vlasov} are given in \citep{Cavallazzi}. 

Finally, we are interested in the distributional properties of the McKean-Vlasov equation \eqref{eq:McKean–Vlasov}. 
% with heavy-tailed diffusion which elicits the compressibility. 
The following result establishes that   the marginal distributions of solutions to \eqref{eq:McKean–Vlasov} will have diverging second-order moments, hence, they will be heavy-tailed.
% the solution to stochatic differential equations driven by $\alpha$-stable diffusions under Assumption \ref{assump:coefficients_regularity}.

\begin{theorem}\label{thm:heavy-tail}
    Let $(\mathrm{L}_t)_{t\ge 0}$ be an $\alpha$-stable process. For any time $t$, let $\theta_t$ be the solution to \eqref{eq:McKean–Vlasov} with initialization $\theta_0$ which is independent of $(\mathrm{L}_t)_{t\ge 0}$ such that $\expec{\|\theta_0\|}<\infty$, then the following holds
    \[
    \expec{\|\theta^\infty_t\|^2} = +\infty.
    \]
\end{theorem}
We remark that the result is weak in the sense that details on the tails of $\theta_t$ with respect to $\alpha$ and $t$ are implicit. However, it renders sufficient for our compressibility result
% in Frobenius norm of the matrix with independent $\theta^\infty_t$-distributed columns as 
in Theorem \ref{thm:main}. Now, having proved all the necessary ingredients, Theorem~\ref{thm:main} is obtained by accumulating the error bounds proven in Theorems~\ref{thm:euler} and \ref{thm:poc}, and applying \citep[Proposition 1]{compressible_distribution} along with Theorem~\ref{thm:heavy-tail}.

% the compressibility of one hidden layer neural networks .... criteria for pruning ??
% In our the experiments, we train two-layer neural networks of increasing widths $n=2\text{k},\ 5\text{k}, 10\text{k}$ to see the convergence phenomena of empirical distribution of parameters using SGD perturbed by stable-noise injection implemented in pytorch.

\paragraph{Additional theoretical results.}

In the Appendix, we investigate two other properties of the considered scheme. In Appendix~\ref{sec:incomp}, we prove that when the injected noise is not heavy-tailed (i.e., $\alpha$ is set to $2$ and the noise becomes Gaussian) and when the step-size goes to zero, the obtained network weights \emph{cannot} be compressible in terms of the notion we defined in Theorem~\ref{thm:main}. This shows that heavy-tails are instrumental in order to guarantee compressibility in our specific compression definition. 

In Appendix~\ref{sec:local}, we investigate the effects of the heavy tails on the training loss. In particular, we upper-bound the expected gradient norm, i.e., $\mathbb{E}\|\nabla R(\hat{\Theta}^n_k)\|^2$ and show that the gradient norm will be bounded by two terms: (i) one term that linearly goes to zero as $K$ increases, (ii) another term, that scales up with the noise scale $\sigma$. This result highlights the fact that injecting heavy-tailed noise introduces a trade-off: while the noise is beneficial in terms of compressibility, it might hurt the optimization performance. In the next section, we investigate this trade-off in different experiments.  

\section{Empirical Results}

% \melihc{Bold and blue text are comments from Melih.}
% \melih{Blue (non-bold) text are significant changes Melih made that you should see.}
% \meliht{Orange texts will be updated when we reestablish server access.}

In this section, we validate our theory with empirical results. Our goal is to investigate the effects of the heavy-tailed noise injection in SGD in terms of compressibility and the train/test performance.  %\utodo{can you prepare the code to be submitted as well? you need to delete all the folder names etc that can reveal our identity}
For our experiments we use the ECG5000 \citep{ECG5000}, MNIST \citep{mnist2010}, CIFAR10, and CIFAR100 \citep{cifar102009} datasets. By slightly stretching the scope of our theoretical framework, we also train the weights of the second layer instead of fixing them to $1/n$. We start our experiments with a single-hidden-layer neural network with ReLU activations and the cross entropy loss, applied on classification tasks. We then examine how well our results generalize to more complex architectures by conducting experiments using fully connected neural networks (FCN) with more hidden layers, as well as using convolutional neural networks (CNN). %Source code for our experiments will be made publicly available upon publication.

\begin{table*}[t]
\centering
\begin{tabular}{ |c|c| c| c| c| c|}
 \hline 
$\alpha$ & Train Acc. & Test Acc. & Pruning Ratio & Train Acc. a.p. & Test Acc. a.p.  \\
 \hline \hline
 no noise & $95.10 \pm 0.14$ & $94.32 \pm 0.27 $& $11.43 \pm 0.04$ & $94.20 \pm 1.13$ & $93.31 \pm 0.28$ \\
 \hline
1.75  & $95.04\pm 0.09$ & $94.19\pm 0.25$ & $50.35\pm 20.63$ & $90.20\pm 7.63$ & $90.32\pm 6.53$ \\ 
\hline
1.8 & $95.04\pm 0.09$ & $94.19\pm 0.25$ & $37.22\pm 20.51$ & $95.16\pm 0.89$ & $93.92\pm 0.54$\\ 
\hline
1.9 & $95.04\pm 0.09$& $94.19\pm 0.25$ & $24.89\pm 9.89$ & $94.84\pm 0.57$ & $94.32\pm 0.46$ \\
\hline
\end{tabular}
\caption{ECG5000, Type-I noise, $n=2$K.}
\label{tbl:ecg_1_2k}
\end{table*}

\begin{table*}[t]
\centering
\begin{tabular}{ |c|c| c| c| c| c|}
 \hline 
$\alpha$ & Train Acc. & Test Acc. & Pruning Ratio & Train Acc. a.p. & Test Acc. a.p.  \\
 \hline \hline
 no noise & $95.36 \pm 0.33$ & $94.30 \pm 0.29 $& $11.44 \pm 0.02$ & $94.12 \pm 3.09$ & $93.53 \pm 2.18$ \\
 \hline
1.75  & $95.36\pm 0.33$ & $94.30\pm 0.29$ & $60.23\pm 17.85$ & $91.16\pm 6.84$ & $90.75\pm 6.72$ \\ 
\hline
1.8 & $95.36\pm 0.33$ & $94.30\pm 0.29$ & $49.63\pm 6.17$ & $94.12\pm 1.43$ & $93.17\pm 1.46$\\ 
\hline
1.9 & $95.36\pm 0.33$& $94.30\pm 0.29$ & $30.05\pm 8.12$ & $94.44\pm 1.06$ & $93.86\pm 1.00$ \\
\hline
\end{tabular}
\caption{ECG5000, Type-I noise, $n=10$K.}
\label{tbl:ecg_1_10k}
\end{table*}

For SGD, the step size is chosen to be small enough to approximate the continuous dynamics given by the McKean-Vlasov equation in order to stay close to the theory, but also not too small so that SGD converges in a reasonable amount of time. We fix the batch size to be as large as possible within memory constraints. For all experiments, the training was continued until reaching 95\% accuracy on the training set.
As for the noise level $\sigma$, we try a range of values for each dataset and $n$, and we chose the largest $\sigma$ such that the perturbed SGD converges. Intuitively, we can expect that smaller $\alpha$ with heavier tails will lead to lower relative compression error. However, it does not guarantee better test performance: we will investigate the trade-offs between compression error and test performance more in detail below. All the experimentation details are given in Appendix~\ref{sec:app:exp}, and our source code includes the relevant implementation details: \url{https://github.com/mbarsbey/implicit-compressibility}.

\subsection{Experiments with ECG5000}
In our first experiment, we consider the ECG5000 dataset and choose the Type-I noise. Our goal is to investigate the effects $\alpha$ and $n$ over the performance. We repeat the experiments $5$ times and report and average and standard deviations in Tables~\ref{tbl:ecg_1_2k}-\ref{tbl:ecg_3_10k}. Here, for different cases, we monitor the training and test accuracies before and after pruning (a.p.), as well as the pruning ratio: the percentage of the weight matrix that can be pruned while keeping the $90\%$ of the squared norm of the original matrix\footnote{The pruning ratio has the same role of $\kappa$, whereas we fix the compression error to $0.1$ and find the largest $\kappa$ that satisfies this error threshold.}.

The results show that, even for a moderate number of neurons $n=2$K, the heavy-tailed noise results in a significant improvement in the compression capability of the neural network (Table \ref{tbl:ecg_1_2k}). For $\alpha =1.9$, we can see that the pruning ratio increases to $24.89\%$, whereas vanilla SGD can only be compressible with a rate $11.43\%$, and the test performance of the pruned model is superior compared to the latter. We also observe that decreasing $\alpha$ (i.e., increasing the heaviness of the tails) results in a better compression rate; yet, there is a tradeoff between this rate and the test performance. In Table~\ref{tbl:ecg_1_10k}, we repeat the same experiment for $n=10$K. We observe that the previous conclusions become even clearer in this case, as our theory applies to large $n$. For the case where $\alpha=1.75$, we obtain a pruning ratio of $60.23\%$ with test accuracy $90.75\%$, whereas for vanilla SGD the ratio is only $11.44\%$ with a test accuracy of $93.53\%$. 

We also investigate the impact of noise type, where we set $n=10$K and use the same setting as in Table~\ref{tbl:ecg_1_10k}. Tables~\ref{tbl:ecg_2_10k}-\ref{tbl:ecg_3_10k} illustrate the results. We observe that the choice of the noise type impacts both compressibility and accuracy. Type-III noise seems to demonstrate a similar pattern to Type-I, while achieving a worse compression rate overall. On the other hand, although Type-II noise bests Type-I in its performance under $\alpha=1.75$, it loses on performance and/or compression in the other two $\alpha$ values. Accordingly, we conclude Type-I noise to achieve a better tradeoff overall, and proceed to the remaining experiments with it.  

\begin{table*}[t]
\centering
\begin{tabular}{ |c|c| c| c| c| c|}
 \hline 
$\alpha$ & Train Acc. & Test Acc. & Pruning Ratio & Train Acc. a.p. & Test Acc. a.p.  \\
 \hline \hline
1.75 & $95.84 \pm 0.55$ & $94.59 \pm 0.41$ & $60.19 \pm 35.37$ & $92.00 \pm 2.92$ & $91.39 \pm 2.57$ \\ 
\hline
1.8 & $95.88 \pm 0.72$ & $94.61 \pm 0.51$ & $39.98 \pm 18.52$ & $93.72 \pm 2.91$ & $92.97 \pm 2.39$\\ 
\hline
1.9 & $95.60 \pm 0.49$ & $94.67 \pm 0.62$ & $31.34 \pm 21.05$ & $93.80 \pm 2.13$ & $93.00 \pm 2.03$ \\
\hline
\end{tabular}
\caption{ECG5000, Type-II noise, $n=10$K.}
\label{tbl:ecg_2_10k}
\end{table*}

\begin{table*}[t]
\centering
\begin{tabular}{ |c|c| c| c| c| c|}
 \hline 
$\alpha$ & Train Acc. & Test Acc. & Pruning Ratio & Train Acc. a.p. & Test Acc. a.p.  \\
 \hline \hline
1.75 & $95.92 \pm 1.01$ & $94.84 \pm 0.81$ & $57.34 \pm 15.70$ & $92.52 \pm 4.57$ & $92.19 \pm 4.21$ \\ 
\hline
1.8 & $96.04 \pm 0.80$ & $94.77 \pm 1.33$ & $45.24 \pm 24.32$ & $93.56 \pm 3.77$ & $92.72 \pm 3.08$\\ 
\hline
1.9 & $95.88 \pm 0.33$ & $94.71 \pm 0.47$ & $26.35 \pm 18.36$ & $94.44 \pm 1.73$ & $94.06 \pm 1.17$ \\
\hline
\end{tabular}
\caption{ECG5000, Type-III noise, $n=10$K.}
\label{tbl:ecg_3_10k}
\end{table*}

\subsection{Experiments with MNIST}
In our next experiment, we consider the MNIST dataset, set $n=5$K and use Type-I noise. Table~\ref{tbl:mnist_1_5k} illustrates the results as the average and the standard deviation of $5$ runs. Similar to the previous results, we observe that the injected noise has a visible benefit on compressibility. When $\alpha=1.9$, our approach doubles the compressibility of the vanilla SGD (from $10.58\%$ to $23.82\%$), while pruned test accuracy decreases only by $\sim1\%$. On the other hand, when we decrease $\alpha$, the pruning ratio goes up to $40.63\%$, while only compromising $\sim3\%$ of pruned test accuracy.

\begin{table*}[t]
\centering
\begin{tabular}{ |c|c| c| c| c| c|}
 \hline

$\alpha$ & Train Acc. & Test Acc. & Pruning Ratio & Train Acc. a.p.  & Test Acc. a.p.\\
\hline\hline
no noise &  $96.32 \pm0.68	 $  & $96.00 \pm0.48	 $  & $10.58 \pm  0.01	$  & $96.30\pm0.67	$  & $95.95 \pm 0.47$  \\
\hline
$1.75 $  & $95.48	\pm 0.20$ & $95.01\pm0.15	 $ & $40.63\pm 8.55	$ & $93.14	\pm 1.54	$ & $92.89\pm1.70$ \\
\hline
$1.8$ & $95.42	\pm0.25	 $  & $94.95\pm0.16	$  &$36.05\pm 6.53	$ & $93.62\pm 1.32	$  & $93.27\pm 1.33
$ \\ 
\hline	
$1.9 $ & $95.88	\pm 0.36	$ & $95.44	\pm 0.24	$ & $23.82	\pm5.89$  & $95.30\pm0.89	 $ & $94.94\pm 0.81$ \\  
\hline
\end{tabular}
    \caption{MNIST, Type-I noise, $n=5$K.}
    \label{tbl:mnist_1_5k}
    \vspace{-5pt}
\end{table*}

\begin{table*}[t]
\centering
\begin{tabular}{ |c|c| c| c| c| c|}
 \hline
$\alpha$ & Train Acc. & Test Acc. & Pruning Ratio & Train Acc. a.p.  & Test Acc. a.p.\\
\hline\hline 
no noise &  $96.52 \pm 0.85$  & $56.71 \pm 0.38$  & $11.60 \pm 0.09$  & $96.13 \pm 0.91$  & $56.31 \pm 0.50$  \\
\hline
$1.75 $  & $ 95.56 \pm 0.24$ & $51.60 \pm 0.22$ & $49.67 \pm 2.30$ & $95.28 \pm 0.23$ & $51.48 \pm 0.27$ \\
\hline
$1.8$ & $ 95.86 \pm 0.36$  & $52.36 \pm 0.31$  &$41.01 \pm 1.36$ & $95.61 \pm 0.51$  & $52.03 \pm 0.27$ \\ 
\hline	
$1.9 $ & $	96.08 \pm 0.21$ & $52.60 \pm 0.41$ & $30.25 \pm 1.95$ & $96.17 \pm 0.21$ & $52.65 \pm 0.36$ \\  
\hline
\end{tabular}
    \caption{CIFAR10, Type-I noise, $n=5$K.}
    \label{tbl:cifar10_1_5k}
\end{table*}

\subsection{Experiments with CIFAR10 and CIFAR100}
We now test our approach with datasets and model architectures that are relatively more realistic in a machine learning setting (see Appendix \ref{sec:app:exp} for full details). First, we conduct experiments with the CIFAR10 dataset using the architecture in the MNIST experiments above, where we set $n=5$K and use Type-I noise. We present our results as the average and standard deviation of 5 runs in Table~\ref{tbl:cifar10_1_5k}. We observe that the results are similarly positive for CIFAR10, where dramatic improvements in compressibility are obtained for a small cost to pruned test performance.

Importantly, in most practical discussions of compressibility \citep{blalock2020state}, it is also desired that the compressed network is \textit{robust} to compression in terms of performance: That is, the pruned network is expected to maintain its test performance in the face of pruning. To compare the networks trained under our approach to vanilla SGD in terms of robustness, we progressively prune more of the columns of each model, and examine the models' test accuracy under increasing pruning ratios (e.g. $\kappa = 0.1, 0.2, \dots$). The results are presented in Figure \ref{fig:cifar10_cifar10cnn_cifar100}'s column \subref{fig:cifar10_b}. Here we plot models' absolute and relative accuracy as a function of pruning ratio, where relative test accuracy refers to the test accuracy of a pruned model in proportion to its unpruned test accuracy. Our findings unequivocally demonstrate the advantage of our approach: Networks trained with heavy-tailed noise (of all three $\alpha$s) are not only more compressible, but are also more robust to pruning in terms of performance. 

\newcommand{\figsize}{0.25}

\begin{figure}[t]
     \centering
     \subfigure{\includegraphics[width=\figsize\textwidth]{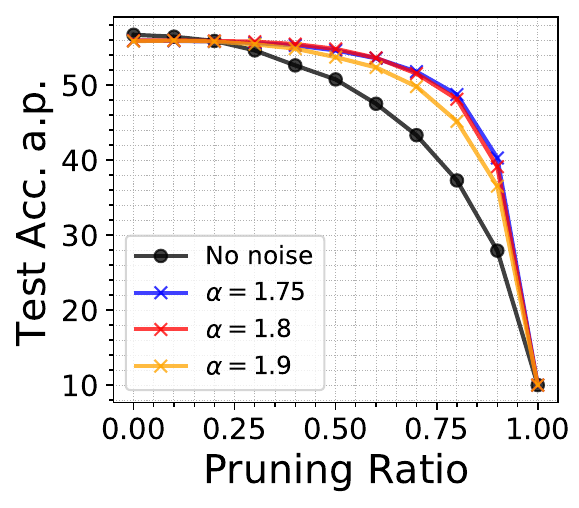}
    \label{fig:cifar10_a}}
\subfigure{\includegraphics[width=\figsize\textwidth]{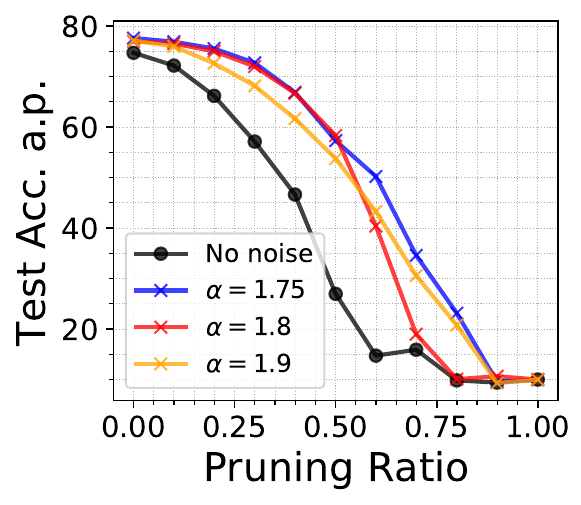}\label{fig:cifar10_cnn_a}}
\subfigure{\includegraphics[width=\figsize\textwidth]{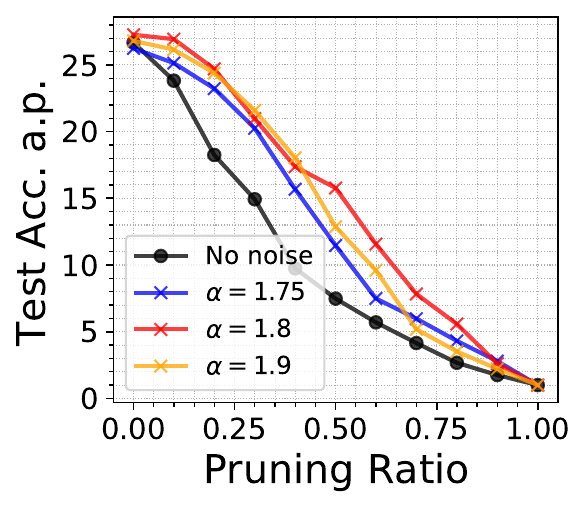}
    \label{fig:cifar100_a}}  
\setcounter{subfigure}{0}
\subfigure[]{\includegraphics[width=\figsize\textwidth]{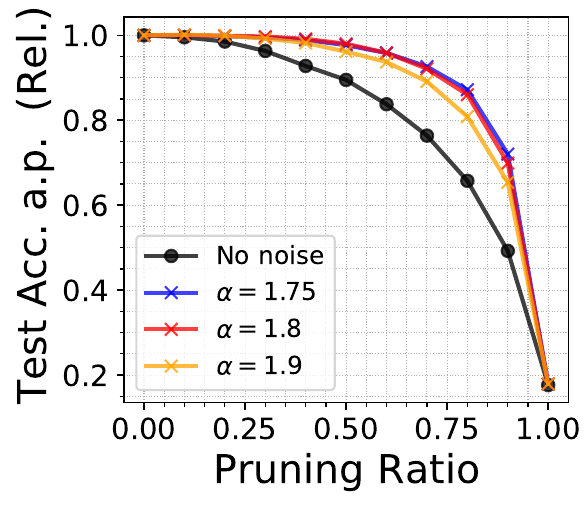}
    \label{fig:cifar10_b}}
     \subfigure[]{\includegraphics[width=\figsize\textwidth]{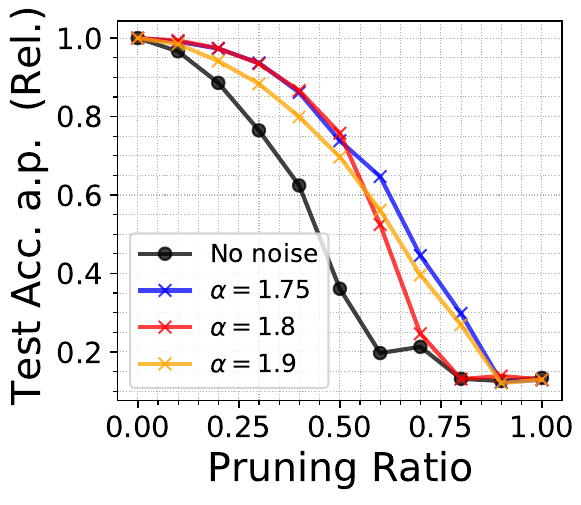}
    \label{fig:cifar10_cnn_b}}     \subfigure[]{\includegraphics[width=\figsize\textwidth]{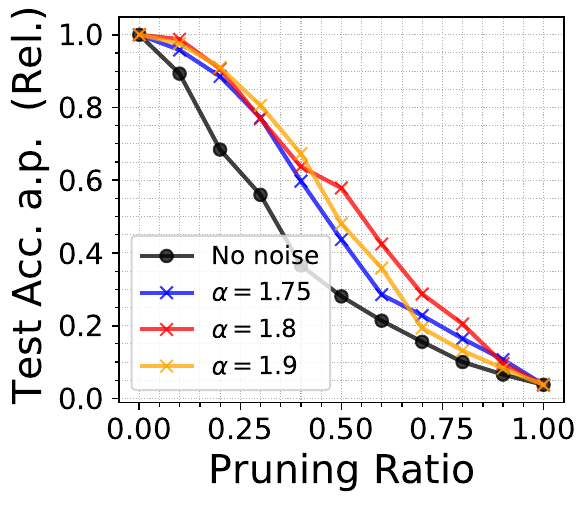}
    \label{fig:cifar100_b}}
    \vspace{-5pt}
     \caption{Absolute and relative test accuracies after pruning, as a function of pruning ratio. Column \subref{fig:cifar10_b}: CIFAR10, FCN with a single hidden layer, $n=5$K. Column \subref{fig:cifar10_cnn_b}: CIFAR10, CNN. Column \subref{fig:cifar100_b}: CIFAR100, FCN with 4 hidden layers, $n=2048$.
     }
\label{fig:cifar10_cifar10cnn_cifar100}

\vspace{-1.25em}
\end{figure}

\paragraph{Robustness with More Complex Architectures}
Inspired by the robustness results presented, we then test whether our results generalize to more complex, and arguably more realistic architectures. Though this means venturing beyond our theoretical setting, we find it crucial from a practical point of view to examine whether our methodology obtains robustness in such contexts. 
%We then test whether our method provides similar benefits when conducting training with other architectures, more specifically convolutional neural networks. 
For this purpose, we train a CNN model, a slightly modified version of the VGG11 model \citep{simonyanVeryDeepConvolutional2015} as described in Appendix \ref{sec:app:exp}, and conduct training on CIFAR10 dataset with noiseless and noise-added networks. The results in column \subref{fig:cifar10_cnn_b} of Figure \ref{fig:cifar10_cifar10cnn_cifar100} again demonstrate the advantage of our approach: the noise-added networks are much more robust to pruning compared to those trained with noiseless SGD.

Lastly, we test our approach using a more challenging classification dataset, CIFAR100. To match the complexity of the task, this time we utilize an FCN with 4 hidden layers of width 2048. We again conduct training until 95\% training accuracy. The results can be seen in column \subref{fig:cifar100_b} of Figure \ref{fig:cifar10_cifar10cnn_cifar100}, and are consistent with the preceding results: noise-added networks are consistently more robust to pruning than their clean-trained counterpart. 

\section{Conclusion}

We provided a methodological and theoretical framework for provably obtaining compressibility in mean-field neural networks. Our approach requires minimal modification for vanilla SGD and has the same computational complexity. By proving discretization error bounds and propagation of chaos results, we showed that the resulting algorithm is guaranteed to provide compressible parameters. We tested our approach through several experiments, where we showed that in most cases the proposed approach achieves high compressibility, while only slightly compromising accuracy. Moreover, we showed that our methodology produces models that are more robust to pruning in terms of test performance, even with architectures that are beyond our theoretical setting, speaking to the promise of our approach from a practical point of view.

The limitations of our approach are as follows: (i) we consider mean-field networks, it would be of interest to generalize our theoretical results to more sophisticated architectures, (ii) adaptive optimizers are frequently used in various fields of machine learning, thus extending our results to such optimization schemes would be beneficial, (iii) an improved understanding of how data distribution, learning rate, noise type, and architecture interact to produce compressibility would help extend our results to various application domains.%, many of which would benefit from such a practical method of obtaining compressibility.%focused on the compressibility; yet, the noise injection also has an effect on the train/test accuracy \melihc{We should either remove or deemphasize both these points now}. Hence, an investigation of the noise injection on the training loss needs to be performed to understand the bigger picture.
% \paragraph{Impact statement.} Due to the mostly theoretical nature of our paper, it does not have a direct negative social impact. Any potential downstream practical applications are to lead to decreased resource consumption for storage and inference.

% \newpage 

\section*{Acknowledgments}
We thank Alain Durmus and Milad Sefidgaran for helpful suggestions. Umut \c{S}im\c{s}ekli's research is supported by the French government under management of
Agence Nationale de la Recherche as part of the “Investissements d’avenir” program, reference
ANR-19-P3IA-0001 (PRAIRIE 3IA Institute) and the European Research Council Starting Grant
DYNASTY – 101039676.

\bibliography{example_paper}

\newpage 

\appendix

\onecolumn

\begin{center}

\Large \bf Implicit Compressibility of Overparametrized Neural Networks via Heavy-Tailed Noisy Gradient Descent \vspace{3pt}\\ {\normalsize APPENDIX}

\end{center}

The Appendix is organized as follows. 
\begin{itemize}
    \item In Section A, we discuss incompressibility under injection of Gaussian noise. 
    \item In Section B, we provide insights about the local convergence properties of the SGD algorithm described in the paper.
    \item In Section C, we provide technical lemmas that are useful for proving Theorem \ref{thm:main}, Theorem \ref{thm:euler} and Theorem \ref{thm:poc}. 
    \item In section D, we provide proofs of the theoretical results in the main paper.
    \item In section E, we present the details for our experiments.
    \item In Section F, implications of our compressibility studies on federated learning are discussed.
\end{itemize}

\section{Incompressibility of the Brownian Case}

\label{sec:incomp}

In this section, we show that injection of Gaussian noise rather than heavy-tailed noise does not result in compressibility in the sense of the notion defined in Theorem~\ref{thm:main}. 
More precisely, consider the following SDE, which is equivalent to  \eqref{eq:SGF} if $\alpha =2$:
\begin{equation}
\label{eq:SGFBrownian}
    % \begin{cases}  
   \diff \theta^{i,n}_t =   b(\theta^{i,n}_t,\mu^n_t) \diff t + \sigma\diff \mathrm{B}^{i,n}_t %\\
    % \theta^{i,n}_0 \sim \mu_0 \in \mathcal{P}(\R^d),
     % \end{cases}
\end{equation}
where $\mathrm{B}^{i,n}_t$,~$i=1,\ldots,n$, denote $n$ independent Brownian motion. The next result shows that the neural network trained according to~\eqref{eq:SGFBrownian} is not compressible, in the sense that there exists at least one value of the compression level $\kappa$ for which the conclusion of Theorem \ref{thm:main} does not hold. 

\begin{proposition}\label{prop:non-compressibility}
Suppose that Assumption~\ref{assump:coefficients_regularity} holds and let $\hat\Theta^n_t \in \mathbb{R}^{p \times n}$ be the matrix with columns being all the parameters $\theta^{i,n}_t$, $i=1,\ldots,n$ obtained by the recursion \eqref{eq:SGFBrownian}. Then, there exists a relative compression error $\varepsilon$ such that for any $\kappa>0$,
\[
\limsup_{n\to\infty} 
\frac{\left\|\hat\Theta^{(\kappa n)}_t - \hat\Theta^n_t \right\|_F}{\left\|\hat\Theta^n_t \right\|_F} > \varepsilon .
\]
\end{proposition}
\begin{proof}
    In the mean-field scaling regime, \citep[Theorem 10]{PoC_gaussian} showed that when the SDEs are driven by Brownian motions, the iterates of SGD have finite second-order moments. Whereas independent samples of Gaussian random variables are not compressible in the sense of Theorem \ref{thm:main}, see \citep[Proposition 1]{compressible_distribution}. This completes the proof.
\end{proof}

\section{Local Convergence of SGD with $\alpha$-Stable Noise} 

\label{sec:local}

In this part, we provide insights about the local convergence properties of the SGD algorithm described by \eqref{eq:SGD}, as guided by \citep[Theorem 5]{csimcsekli2019heavy}.

\begin{proposition}
Let $n$ be the number of neurons in the one-hidden-layer neural network and recall that $\hat{\Theta}^n_k$ represents the matrix with columns being the individual neuron weights at iteration $k$ of the SGD described by \eqref{eq:SGD}. Under Assumption \ref{assump:coefficients_regularity}, for some $0<\gamma < \alpha-1$ and if $\sigma \eta^{1/\alpha-1}>1$, we have
\begin{align*}
&\min_{0 \leq k \leq K-1} \mathbb{E}\|\nabla R(\hat{\Theta}^n_k)\|^2 
\le \frac{R(\hat\Theta^n_0) - R^*}{Kn \eta} + \frac{(2B)^{2+\gamma}\sigma^{1+\gamma}}{(1+\gamma)n^2} \eta^{\frac{\gamma+1-\alpha}{\alpha}}.
\end{align*}
If particular, if $\eta$ is chosen such that $\eta\in (n^{2\alpha/(\gamma+1-\alpha)+\epsilon}, \sigma^{\alpha/(\alpha-1)})$ for some $\epsilon>0$ small enough, the upper bound goes to $0$ as $K$ and $n$ go to infinity.
\end{proposition}

\begin{proof}
 It suffices to show that A3 and A7 in \citep{csimcsekli2019heavy} holds with stepsize $\eta n$, $M=B/n$ (in A3) and $\sigma^{1+\gamma}_{\gamma} =  2(2B/n)^{1+\gamma}(\sigma \eta^{1/\alpha-1})^{1+\gamma}$ (in A7), where $B$ is the constant in Assumption \ref{assump:coefficients_regularity}.
 
 Since every bounded Lipschitz-function is $\gamma$-Holder continuous for every $\gamma\in (0,1]$, the $\gamma$-Holderness A3 follows from Lemma \ref{lm:Lip_coeffi} follows directly from Assumption \ref{assump:coefficients_regularity} by taking $\gamma\in (0,\alpha -1)$. It is obvious that the noise $\sigma \eta^{1/\alpha-1} X^{i,n}_k$ of the noisy stochastic gradient noise in \eqref{eq:SGD} is unbiased. Then it remains to verify that the gradient descent updates satisfy certain moment bounds. To this end, note that we have
\begin{align*}
&\expec{\|\nabla R(\Theta^n_k) - \sigma \eta^{1/\alpha-1}X^{i,n}_k/n\|^{1+\gamma}|\hat\Theta^n_k} \\ \le
&2^{1+\gamma}\expec{\nabla R(\Theta^n_k) \|^{1+\gamma}|\hat\Theta^n_k} + 2^{1+\gamma}\expec{\|\sigma \eta^{1/\alpha-1}X^{i,n}_k/n\|^{1+\gamma}
 } \\ \le
& (2B)^{1+\gamma}(\sigma \eta^{1/\alpha-1}/n)^{1+\gamma} + (2B/n)^{1+\gamma} \\ \le
& 2(2B/n)^{1+\gamma}(\sigma \eta^{1/\alpha-1})^{1+\gamma}
\end{align*}
Finally, using~\citep[Theorem 5]{csimcsekli2019heavy} completes the proof.
\end{proof}

\section{Technical Lemmas}

\begin{lemma}\label{lm:Lip_coeffi}
   Under Assumption \ref{assump:coefficients_regularity}, it holds that
    \[
    \|b(\theta_1,\mu_1) - b(\theta_2,\mu_2)\| \le 
    B \cdot \big(\|\theta_1-\theta_2\| + \mathbb{E}_{x\sim\pi}\left[|\mu_1(h_x(\cdot)) - \mu_2(h_x(\cdot))|^2\right]^{\frac{1}{2}}\big).
    \]
    Moreover, $\|b(\cdot,\cdot)\|_\infty \le B$, and if $\mu_1 = \frac{1}{n}\sum_{i=1}^n\delta_{\theta^i_1}$, $\mu_2 = \frac{1}{n}\sum_{i=1}^n\delta_{\theta^i_2}$,
    \[
     \|b(\theta_1,\mu_1) - b(\theta_2,\mu_2)\| \le B\|\theta_1 - \theta_2\| + \frac{B}{n}\sum_{i=1}^n \|\theta^i_1 - \theta^i_2\|.
    \]
\end{lemma}

\begin{proof} 
Recall that
    \[
    b(\theta,\mu) = -\expec{\partial_1 l(\mu(h_x(\cdot)),y)\nabla h_x(\theta)}.
    \]
Then it follows from the triangular inequality that
\begin{equation}\label{eq:Lip_1}
    \|b(\theta_1,\mu_1) - b(\theta_2,\mu_2)\| 
    \le  \|b(\theta_1,\mu_1) - b(\theta_2,\mu_1)\| + \|b(\theta_2,\mu_1) - b(\theta_2,\mu_2)\|.
\end{equation}
The first term is upper bounded by
\begin{equation}\label{eq:Lip_2}
\begin{aligned}
  \|b(\theta_1,\mu_1) - b(\theta_2,\mu_1)\| \le
   & \expec{\|\partial_1 l(\cdot,y)\|_\infty \cdot \|\nabla^2 h_x\|_\infty }\cdot\|\theta_2-\theta_1\| \\ \le
   & \expec{\Phi(y)\Psi(x)} \cdot \|\theta_1 - \theta_2\| \\ \le
   & \left(\expec{\Phi^2(y)\Psi^2(x)}\right)^{1/2}\cdot \|\theta_1 - \theta_2\| \\ \le
   & B \cdot \|\theta_1-\theta_2\|.
\end{aligned}
\end{equation}
The second term is upper bounded by
\begin{equation}\label{eq:Lip_3}
\begin{aligned}
\|b(\theta_2,\mu_1) - b(\theta_2,\mu_2)\|  \le
& \expec{\| \partial^2_1 l(\cdot,y)\|_\infty \cdot \|\nabla h_x(\cdot) \|_\infty \cdot |\mu_1(h_x(\cdot)) - \mu_2(h_x(\cdot)) |}\\  \le
& \left(\expec{\Phi^2(y)\Psi^2(x)}\right)^{1/2} \expec{|\mu_1(h_x(\cdot)) - \mu_2(h_x(\cdot)) |^2}^{1/2} \\ \le
& B \cdot \expec{|\mu_1(h_x(\cdot)) - \mu_2(h_x(\cdot)) |^2}^{1/2}.
\end{aligned}
\end{equation}
The first inequality then follows by combining \eqref{eq:Lip_1}, \eqref{eq:Lip_2} and \eqref{eq:Lip_3}.

For the boundedness of $b$ in the norm infinity, it is not difficult to see that
\begin{align*}
b(\theta,\mu) &= -\mathbb{E}[\partial_1 l(\mu(h_x(\cdot)), y)\nabla h_x(\theta)]  \\
&\le \expec{\Phi(y)\Psi(x)} \\
&\le B.
\end{align*}

 The proof of the last inequality follows by using  the first bound and the Cauchy-Schwarz inequality as
\begin{align*}
\|b(\theta_1,\mu_1)-b(\theta_2,\mu_2)\| 
\le & B\|\theta_1 - \theta_2\| + \frac{1}{n} \mathbb{E}_{x\sim\pi}\left[\left(\sum_{i=1}^n h_x(\theta^i_1) - h_x(\theta^i_2) \right)^2 \right]^{1/2}   \\
\le & B\|\theta_1-\theta_2\| + \frac{1}{n}\mathbb{E}_{x\sim\pi}\left[\|\nabla h_x\|_\infty \left(\sum_{i=1}^n \|\theta^i_1-\theta^i_2\| \right)^2 \right]^{1/2} \\
\le & B\|\theta_1 - \theta_2\| + \frac{1}{n}\mathbb{E}_{x\sim\pi}[\Psi^2(x)]^{1/2} \cdot \sum_{i=1}^n \|\theta^i_1-\theta^i_2\|\\
\le & B\|\theta_1-\theta_2\| + \frac{B}{n}\sum_{i=1}^n \|\theta^i_1 - \theta^i_2\|.
\end{align*}
This completes the proof.
\end{proof}

\subsection{Propagation of Chaos}

\begin{lemma}\label{lem:integrability}
Let $(\mathrm{L}_t)_{t\ge 0}$ be an $\alpha$-stable L\'evy process and let $(\mathcal{F}_t)_{t\ge 0}$ be the filtration generated by $(\mathrm{L}_t)_{t\ge 0}$. Then under Assumption \ref{assump:coefficients_regularity}, given the the initial condition $X_0=\xi$, there exists a unique adapted process $(X_t)_{t\in[0,T]}$ for all integrable datum $\xi \in L^1(\R^p)$ such that
\[
X_t = \xi + \int_0^t b(X_t,[X_t])\diff t + \mathrm{L}_t.
\]
Moreover the first moment of the supremum of the process is bounded 
\[
\expec{\sup_{t\le T}\| X_t \|} < +\infty.
\]
\end{lemma}

\begin{proof}
 The proof follows from Theorem 1 in \citep{Cavallazzi} by Lemma \ref{lm:Lip_coeffi} where $\beta$ is set to $1$.
\end{proof}

\subsection{Compression}

\begin{lemma}\label{lm:avr_inf}
Consider a non-integrable probability distribution $\mu$ taking values in $\mathbb{R}_+$ such that $\mathbb{E}_{X\sim \mu}[X]=+\infty$. Let $X_1,\ldots,X_n$ be $n$ i.i.d. copies distributed according to $\mu$. Then for any $C$ positive,
\[
\mathbb{P}\left[ \frac{1}{n} \sum_{i=1}^n X_i \le C \right] \xrightarrow{n\to\infty} 0. 
\]
\end{lemma}

\begin{proof}
    Using the assumption that $\mu$ is non-integrable, let $K$ be a cutoff level for $\mu$ such that
    \[
    \mathbb{E}_{X\sim\mu}[\max (X
    ,K)] = C+1.
    \]
    Therefore by the law of large numbers, when goes to infinity,
    \[
    \lim_{n\to\infty}\frac{1}{n}\sum_{i=1}^n \max(X_i,K) = C+1 \qquad \text{almost surely}.
    \]
    Finally, note that
    \[  \frac{1}{n}\liminf_{n\to\infty}\sum_{i=1}^n X_i \ge \frac{1}{n}\lim_{n\to\infty} \sum_{i=1}^n \max(X_i,K), 
    \]
    which is lower bounded by $(C+1)$ almost surely. Thus the probability that $\frac{1}{n}\sum_{i=1}^n X_i$ be smaller than $C$ vanishes for large (infinite) values of $n$.
\end{proof}

\section{Proofs}

\subsection{Proof of Theorem~\ref{thm:heavy-tail}}
\begin{proof}
Recall that $\theta_t = \theta_0 + \int_0^t b(\theta_s,[\theta_s])\diff s + \mathrm{L}_t$, then
    \begin{align*}
        \expec{\|\theta_t\|^2}  =
        & \expec{\left\langle\theta_0 + \int_0^t b(\theta_s,[\theta_s])\diff s + \mathrm{L}_t, \theta_0 + \int_0^t b(\theta_s,[\theta_s])\diff s + \mathrm{L}_t\right\rangle}\\ =
        & \expec{\left\|\theta_0+ \int_0^t b(\theta_s,[\theta_s]) \diff s\right\|^2} + 2\expec{\left\langle \theta_0+ \int_0^t b(\theta_s,[\theta_s]) \diff s, \mathrm{L}_t\right\rangle} \\
        &+ \expec{\|\mathrm{L}_t\|^2}\\ \ge
        & \expec{\|\mathrm{L}_t\|^2} - 2\expec{ \|\theta_0\| \cdot \|\mathrm{L}_t\|} - 2\expec{t\|b(\cdot)\|_\infty \cdot \|\mathrm{L}_t\|}  \\ \ge
        & \expec{\|\mathrm{L}_t\|^2} - 2\expec{ \|\theta_0\|} \expec{\|\mathrm{L}_t\|} - 2Bt \cdot\expec{\|\mathrm{L}_t\|},
    \end{align*}
where the last inequality follows from the independence between the initialization $\theta_0$ and the diffusion noise $(\mathrm{L}_t)_{t\ge 0}$ and by using Lemma \ref{lm:Lip_coeffi}. The proof is completed by noticing that
\[
\expec{\|\mathrm{L}_t\|^2} = \infty \quad \text{ and } \quad \expec{\|\theta_0\|}, \expec{\|\mathrm{L}_t\|} < \infty.
\]
\end{proof}

\subsection{Proof of Theorem~\ref{thm:poc}}
\begin{proof} 
By identification of the diffusion process $(\mathrm{L}^{i,n}_t)_{t\ge 0}$ in \eqref{eq:SGF} and \eqref{eq:McKean–Vlasov}, the difference of their solutions $\theta^{i,n}_t$ and $\theta^{i,\infty}_t$ for all $t\in[0,T]$ satisfies
\[
\theta^{i,n}_t - \theta^{i,\infty}_t = \int_0^t [b(\theta^{i,n}_s,\mu^n_s) - b(\theta^{i,\infty}_s, [\theta^{i,\infty}_s]) ]\diff s,
\]
where $\mu_t = \frac{1}{n}\sum\limits_{i=1}^n \delta_{\theta^{i,n}_t}$ and $[\theta^{i,\infty}_t]$ denotes the distribution of $\theta^{i,\infty}_t$. Using Lemma \ref{lm:Lip_coeffi},
\begin{equation}\label{eq:PoC_1}
\begin{aligned}
     \|\theta^{i,n}_t - \theta^{i,\infty}_t\| 
    \le & B\int_0^t \|\theta^{i,n}_s-\theta^{i,\infty}_s\|\diff s + B \int_0^t \mathbb{E}_{x\sim\pi}\big[|\mu^n_s(h_x(\cdot)) - [\theta^{i,\infty}_s](h_x(\cdot))|^2\big]^{1/2} \diff s \\
    \le & B \int_0^t \sup_{i\le n}\|\theta^{i,n}_s - \theta^{i,\infty}_s\| \diff s + B\int_0^t \mathbb{E}_{x\sim\pi}\left[|\mu^n_s(h_x(\cdot)) - \bar\mu^n_s(h_x(\cdot))|^2\right]^{1/2} \diff s\\
    & + B\int_0^t \mathbb{E}_{x\sim\pi}\left[|\bar\mu^n_s(h_x(\cdot)) - [\theta^{i,\infty}_s](h_x(\cdot))|^2\right]^{1/2} \diff s
\end{aligned}
\end{equation}
where $\bar\mu^n_s := \frac{1}{n}\sum\limits_{i=1}^n \delta_{\theta^{i,\infty}_s}$, the empirical measure of $\theta^{i,\infty}_s$ for $i = 1,\ldots,n$. the last inequality follows from Cauchy-Schwarz inequality.
 Moreover we have
\begin{align*}
\mathbb{E}_{x\sim\pi}[|\mu^n_s(h_x(\cdot)-\bar\mu^n_s(h_x(\cdot))|^2]^{1/2}  \le & \mathbb{E}_{x\sim\pi}\left[\left|\frac{\|\nabla h_x\|\infty}{n} \sum_{i=1}^n \|\theta^{i,n}_s - \theta^{i,\infty}_s\|\right|^2\right]^{1/2} \\
\le & \mathbb{E}_{x\sim\pi}[\Psi^2(x)] ^{1/2}\cdot \frac{1}{n}\sum_{i=1}^n \|\theta^{i,n}_s - \theta^{i,\infty}_s\| \\
\le & B\sup_{i\le n}\|\theta^{i,n}_s - \theta^{i,\infty}_s\|.
\end{align*}
Plugging the above estimate into \eqref{eq:PoC_1} yields
\begin{equation}\label{eq:PoC_2}
   \|\theta^{i,n}_t - \theta^{i,\infty}_t\|  \le  B(1+B)\int_0^t \sup_{i\le n}\|\theta^{i,n}_s - \theta^{i,\infty}_s\| \diff s + B\int_0^t \mathbb{E}_{x\sim\pi}\left[|\bar\mu^n_s(h_x(\cdot)) - [\theta^{i,\infty}_s](h_x(\cdot))|^2\right]^{\frac{1}{2}} \diff s.
\end{equation}

Taking the supremum over $i=1,\ldots,n$ and t, and using the fact that 
\[
\sup\int_\cdot(\cdot) \le \int_\cdot \sup(\cdot),
\]
we get
\begin{equation}\label{eq:PoC_3}
\begin{aligned}
\sup_{t\le T}\sup_{i\le n}\|\theta^{i,n}_t - \theta^{i,\infty}_t\| \le  & B(1+B)\int_0^T \sup_{t\le s}\sup_{i\le n}\|\theta^{i,n}_t - \theta^{i,\infty}_t\|\diff s \\
& + B\int_0^t \mathbb{E}_{x\sim\pi}\left[|\bar\mu^n_s(h_x(\cdot)) - [\theta^{i,\infty}_s](h_x(\cdot))|^2\right]^{1/2} \diff s.
\end{aligned}
\end{equation}

Let us now estimate $\expec{|\bar\mu^n_s(h_x(\cdot)) - [\theta^{i,\infty}_s](h_x(\cdot))|^2 | x}^{1/2}$, the expectation under the stable diffusion, rather than the expectation over the data distribution, where the $1/\sqrt{n}$ convergence rate comes from. Indeed for fixed $x$, $h_x(\theta^{i,\infty}_s)$, $i=1,\ldots,n$ are bounded i.i.d. random variables with mean value $[\theta^{i,\infty}_s](h_x(\cdot))$. Therefore 
\begin{equation}\label{eq:PoC_4}
\begin{aligned}
    \expec{|\bar\mu^n_s(h_x(\cdot)) - [\theta^{i,\infty}_s](h_x(\cdot))|^2 | x}^{1/2} =
    & \expec{\left.\left|\frac{1}{n}\sum_{i=1}^n h_x(\theta^{i,\infty}_s) - [\theta^{i,\infty}_s](h_x(\cdot))\right|^2 \right| x}^{1/2} \\ \le
    & \frac{1}{\sqrt{n}} \|h_x(\cdot)\|_\infty  \le \frac{\Psi(x)}{\sqrt{n}}.
\end{aligned}
\end{equation}
Finally, combining \eqref{eq:PoC_3}, \eqref{eq:PoC_4}, the integrability condition Lemma \ref{lem:integrability} and using Fubini's theorem, we get
\[
\expec{\sup_{r\le t}\sup_{i\le n}\|\theta^{i,n}_r - \theta^{i,\infty}_r\|} 
\le  B(1+B)\int_0^t \expec{\sup_{r\le s}\sup_{i\le n}\|\theta^{i,n}_r - \theta^{i,\infty}_r\|} \diff s  + \frac{Bt\mathbb{E}_{x\sim\pi}[\Psi(x)]}{\sqrt{n}}.
\]
Finally, by Gronwall's inequality we get
\[
\expec{\sup_{t\le T}\sup_{i\le n}\|\theta^{i,n}_t - \theta^{i,\infty}_t\|} 
\le (1+B)\left(\frac{BT}{\sqrt{n}} + \frac{B^2T^2\exp(B T(1+\mathbb{E}_{x\sim\pi}[\Psi(x)]))}{2\sqrt{n}} \right).
\]
This completes the proof of Theorem~\ref{thm:poc}.
\end{proof}

\subsection{Proof of Theorem~\ref{thm:euler}}
\begin{proof}
 Similar to in the proof of Theorem~\ref{thm:poc}, we have
\begin{align*}
    \sup_{i\le n}\|\theta^{i,n}_\eta - \hat\theta^{i,n}_1\|    \le 
    & \sup_{i\le n}\int_0^\eta\| b(\theta^{i,n}_t,\mu^{n}_t) - b(\hat\theta^{i,n}_0, \mu^{n}_0) \| \diff t \\ \le
    &  B\int_0^\eta \sup_{i\le n}\| \theta^{i,n}_t - \theta^{i,n}_0\| + \frac{1}{n} \sum_{j=1}^n \| \theta^{j,n}_t - \theta^{j,n}_0 \| \diff t \\ \le
    &  B\int_0^\eta 2\|b\|_\infty \cdot t + \sup_{i\le n}\|\mathrm{L}^{i,n}_t\| + \frac{1}{n}\sum_{j=1}^n \|\mathrm{L}^{j,n}_t\| \diff t 
\end{align*}
Recall that $\|b\|_\infty\le B$, therefore by taking the expectation and the scaling of the stable process $\mathrm{L}^{i,n}_t$, we get
\begin{equation}\label{eq:error_est_1step}
\begin{aligned}
    \expec{\sup_{i\le n}\|\theta^{i,n}_\eta - \hat\theta^{i,n}_1\|}  \le 
    &B \int_0^\eta \left(2Bt + t^{1/\alpha}\cdot \expec{\sup_{i\le n}\|\mathrm{L}^{i,n}_1\| + \frac{1}{n}\sum_{j=1}^n \|\mathrm{L}^{j,n}_1 \|}\right)\diff t \\ \le
    & B^2\eta^2+ \frac{B\alpha \cdot \expec{\sup_{i\le n}\|\mathrm{L}^{i,n}_1\| + \|\mathrm{L}^\alpha_1\|} }{\alpha+1} \eta^{1+1/\alpha}.
\end{aligned}
\end{equation}
Denote by $C' := \expec{\sup_{i\le n}\|\mathrm{L}^{i,n}_1\| + \|\mathrm{L}^\alpha_1\|}$ and $\psi_t(\xi)$ the solution of \eqref{eq:SGF} at time $t$ with initial condition $\xi \in \R^{n \times d}$, which is the matrix of $n$ vectors $\psi^{i,n}_t(\xi) \in \R^d$, $i=1, \ldots, n$. At time $T$ which is a multiple of $\eta$,
\begin{equation}\label{eq:error_est_sum}
     \theta^{i,n}_T - \hat\theta^{i,n}_{T/\eta} = \sum_{k=0}^{T/\eta-1} \psi^{i,n}_{T-\eta k}(\hat\Theta^n_k) - \psi^{i,n}_{T-\eta(k+1)}(\hat\Theta^n_{k+1}),
\end{equation}
where $\hat\Theta^n_k$ is the matrix of $\hat\theta^{i,n}_k$. Similarly, for each of the terms inside the summation above,
\begin{equation}\label{eq:error_est_pf_1}
\begin{aligned}
    \psi^{i,n}_{T-\eta k}(\hat\Theta^n_k) - \psi^{i,n}_{T-\eta(k+1)} (\hat\Theta^n_{k+1}) 
    =& \left[ \int_{\eta k}^{\eta (k+1)} b^{i,n}(\psi_{t-\eta k}(\hat\Theta^n_k))\diff t + \diff\mathrm{L}^{i,n}_t  - (\hat\theta^{i,n}_{k+1}-\hat\theta^{i,n}_k) \right] \\
    & - \int_{\eta(k+1)}^T \left(b^{i,n}(\psi_{t-\eta k}(\hat\Theta^n_k)) - b^{i,n}(\psi_{t-\eta(k+1)}(\hat\Theta^n_{k+1}))\right)\diff t .
\end{aligned}
\end{equation}
Note that the first term in the big bracket is the difference of one-step increment started from $\hat\Theta^n_k$. It follows from \eqref{eq:error_est_1step} that
\begin{equation}\label{eq:error_est_pf_2} 
\begin{aligned}
\expec{\sup_{i\le n}\left\|\int_{\eta k}^{\eta (k+1)} b^{i,n}(\psi_t(\hat\Theta^n_k)) \diff t + \diff\mathrm{L}^{i,n}_t  - (\hat\theta^{i,n}_{k+1}-\hat\theta^{i,n}_k) \right\|}
\le &  B^2\eta^2+ \frac{B\alpha \cdot C'}{\alpha+1} \eta^{1+1/\alpha}. 
\end{aligned}
\end{equation}
For the second integral term, similarly we have
\begin{equation}\label{eq:error_est_pf_3}
\begin{aligned}
&\expec{\sup_{i\le n}\|b^{i,n}(\psi_{t-\eta k}(\hat\Theta^n_k)) - b^{i,n}(\psi_{t-\eta(k+1)}(\hat\Theta^n_{k+1}))\|} \\
\le & B \cdot \expec{\sup_{i\le n}\|\psi^{i,n}_{t-\eta k}(\hat\Theta^n_k)-\psi^{i,n}_{t-\eta (k+1)}(\hat\Theta^n_{k+1})\|} + \frac{B}{n}\sum_{j=1}^n \expec{\|\psi^{j,n}_{t-\eta k}(\hat\Theta^n_k)-\psi^{j,n}_{t-\eta (k+1)}(\hat\Theta^n_{k+1})\|}
\end{aligned}
\end{equation}
Combining \eqref{eq:error_est_pf_1}, \eqref{eq:error_est_pf_2} and~\eqref{eq:error_est_pf_3} we get
\begin{align*}
    & \mathbb{E}\left[\sup_{i\le n}\|\psi^{j,n}_{T-\eta k}(\hat\Theta^n_k)-\psi^{j,n}_{T-\eta (k+1)}(\hat\Theta^n_{k+1})\| \right] \\
    \le & B^2\eta^2+ \frac{2B\alpha \cdot C'}{\alpha+1} \eta^{1+1/\alpha} +  2B\cdot \int_{\eta(k+1)}^T \expec{\sup_{i\le n}\|\psi^{i,n}_{t-\eta k}(\hat\Theta^n_k)-\psi^{i,n}_{t-\eta (k+1)}(\hat\Theta^n_{k+1})\|} \diff t.
\end{align*}
Next it follows from Gronwall's inequality that 
\[
\expec{\sup_{i\le n}\|\psi^{j,n}_{T-\eta k}(\hat\Theta^n_k)-\psi^{j,n}_{T-\eta (k+1)}(\hat\Theta^n_{k+1})\|}
 \le  \exp(2BT)\left(B^2\eta^2+ \frac{2B\alpha \cdot C'}{\alpha+1} \eta^{1+1/\alpha}\right).
\]
Finally, combining with \eqref{eq:error_est_sum} we obtain
\[
   \expec{ \sup_{i\le n}\|\theta^{i,n}_T - \hat\theta^{i,n}_{T/\eta}\|}
   \le  T\exp(2BT)\left(B^2\eta+ \frac{2B\alpha \cdot C'}{\alpha+1}\eta^{1/\alpha} \right).
\]
Then it follows by Lemma \ref{lm:max_n_stable} that for some constant $C'_\alpha$ that depends on $\alpha$, we have
\[
C' = \expec{\sup_{i\le n}\|\mathrm{L}^{i,n}_1\| + \|\mathrm{L}^\alpha_1\|} \le C'_\alpha (n^{1/\alpha}+1).
\]
This completes the proof of Theorem \ref{thm:euler}.
\end{proof}

\begin{lemma}\label{lm:max_n_stable}
 Take $n$ i.i.d. $\alpha$-stable random variables $X^i$ such that there exists $C_\alpha >0$, for $t$ sufficiently large and $i=1,\ldots,n$, $\mathbb{P}[\|X^i\|\ge t] \ge C_\alpha t^{-\alpha}.$ If $1<\alpha<2$, then there exists $C'_\alpha$ such that for $n$ sufficiently large,
    \[
     \expec{\sup_{i\le n}\|
     X^i \|} \le  C'_\alpha n^{1/\alpha}
    \]
\end{lemma}

\begin{proof}
It is not difficult to see from the condition $
    \mathbb{P}[\|X^i\|\ge t] \ge C_\alpha t^{-\alpha}$ that for large $t$, it holds that
\[
 \mathbb{P}\left[\sup_{i\le n}\|X^i\| \ge t\right] = 1 -\prod_{i=1}^n \mathbb{P}[\|X^i\| <t] \le
 1- \left( 1-C_\alpha t^{-\alpha}\right)^n\le C_\alpha n t^{-\alpha}.
\]
Then, for large $n$ we get
\begin{align*}
     \expec{\sup_{i\le n}\|X^i\|}
    =&\int_0^\infty \mathbb{P}\left[\sup_{i\le n}\|X^i\|\ge t \right] \mathrm{d}t \\
    =& \sum_{k=-1}^{-\infty}\int_{(n/2^{k+1})^{1/\alpha}}^{(n/2^k)^{1/\alpha}} \mathbb{P}\left[\sup_{i\le n}\|X^i\|\ge t \right] \mathrm{d}t + \int_0^{n^{1/\alpha}} \mathbb{P}\left[\sup_{i\le n}\|X^i\|\ge t \right] \mathrm{d}t \\ 
    \le&  n^{1/\alpha} \sum_{k=-1}^{-\infty} 2^{-k/\alpha}\mathbb{P}\left[\sup_{i\le n}\|X^i\| \ge (n/2^{k+1})^{1/\alpha}\right] + n^{1/\alpha}\\
    \le&  C_\alpha n^{1/\alpha} \sum_{k=-1}^{-\infty} 2^{k+1-k/\alpha} + n^{1/\alpha}\\
    \le& C'_\alpha n^{1/\alpha}
\end{align*}
where in the last inequality we set $C'_\alpha=1+2^{1+1/\alpha}/(2-2^{1/\alpha})$. This completes the proof of Lemma \ref{lm:max_n_stable}.
\end{proof}

\subsection{Proof of Theorem~\ref{thm:main}}

\begin{definition}[$k$-term approximation error \citep{compressible_distribution}] 
    The best $k$-term approximation error $\overline{\sigma}_k(\mathbf{x})$ of a vector $\mathbf{x}$ is defined by
    \[
    \sigma_k(\mathbf{x}) = \inf_{\|\mathbf{y}\|_0\le k} \|\mathbf{x} - \mathbf{y}\|,
    \]
    where $\|\mathbf{y}\|_0$ is the $l^0$-norm of $\mathbf{y}$, which counts the non-zero coefficients of $\mathbf{y}$. Without mentioned explicitly, $\|\mathbf{x}\|$ denotes the square norm of $\mathbf{x}$.
\end{definition}

% $\eta \le n^{-\alpha/2-1}$
% $\eta k\le T$:
% \[
% \expec{\sup_{i\le n} \|\theta^{i,n}_t-\hat\theta^{i,n}_{\lfloor t/\eta \rfloor}\|} \le C (\eta n)^{1/\alpha}.
% \]

\begin{proof} 
Denote by $\hat{\mathbf{w}}^n_t = (\|\hat\theta^{1,n}_{\lfloor t/\eta \rfloor}\|, \ldots, \|\hat\theta^{n,n}_{\lfloor t/\eta \rfloor}\|)$ and $\mathbf{w}^*_t = 
 (\|\theta^{1,\infty}_t\|, \ldots, \|\theta^{n,\infty}_t\|)$, where the components $\theta^{i,\infty}_t$ are independent solutions to \eqref{eq:McKean–Vlasov} in Theorem \ref{thm:poc}. Note that the definition of Frobenius matrix norm $\|\cdot\|_F$ gives that
\begin{equation}\label{eq:Frob_norm}
\|\hat\Theta^{\{\kappa n\}}_{\lfloor t/\eta \rfloor} - \hat\Theta^n_{\lfloor t/\eta \rfloor}\|_F = \|\sigma_{\langle \kappa n \rangle}(\hat{\mathbf w}^n_t)\|, \quad \|\hat\Theta^n_{\lfloor t/\eta \rfloor} \|_F = \|\mathbf w^\star_t\|,
\end{equation}
Therefore it suffices to prove Theorem \ref{thm:main} for $\hat{\mathbf{w}}^n_t$. It follows from Theorem \ref{thm:poc} and Theorem \ref{thm:euler} that there exists a constant C independent of $n$ for which
\[
\expec{\sup_{i\le n} \|\hat\theta^{i,n}_{\lfloor t/\eta \rfloor} - \theta^{i,\infty}_t\|} \le \frac{C}{3\sqrt{n}}  
\]
Then by the Markov's inequality we get
\begin{equation}\label{eq:PoC_proba}
\mathbb{P}\left[\sup_{i\le n}\|\hat\theta^{i,n}_{\lfloor t/\eta \rfloor} - \theta^{i,\infty}_t\| > \frac{C}{\epsilon\sqrt{n}}\right] \le \epsilon/3.
    \end{equation}
Denote by $E$ the event 
\[
E := \left\{\sup_{i\le n}\|\hat\theta^{i,n}_{\lfloor t/\eta \rfloor} - \theta^{i,\infty}_t\| \le \frac{C}{\epsilon\sqrt{n}} \right\}.
\]
If $\sup_{i\le n}\|\hat\theta^{i,n}_{\lfloor t/\eta \rfloor} - \theta^{i,\infty}_t\| \le \frac{C}{\epsilon\sqrt{n}}$ and $\|\sigma_{\lfloor\kappa n\rfloor}(\hat{\mathbf{w}}^n_t)\| \ge \epsilon\|\hat{\mathbf{w}}^n_t\|$,  we obtain
\begin{align*}
\|\sigma_{\lfloor \kappa n \rfloor}(\mathbf{w}^\star_t)\| 
&\ge \|\sigma_{\lfloor \kappa n \rfloor}(\hat{\mathbf{w}}^n_t)\| - \kappa n\frac{C}{\epsilon\sqrt{n}} \\
&\ge \epsilon\|\hat{\mathbf{w}}^n_t\|- C\sqrt{n}\kappa/\epsilon \\
&\ge \epsilon(\|\mathbf{w}^\star_t\| -  C\sqrt{n}/\epsilon) - C\sqrt{n}\kappa /\epsilon \\
&= \epsilon \|\mathbf{w}^\star_t\| -  C\sqrt{n}(1 + \kappa/\epsilon)
\end{align*}
Therefore plugging in \eqref{eq:PoC_proba}, we get
\begin{equation}\label{eq:compression_1}
\begin{aligned} 
&\mathbb{P}\big[\|\sigma_{\lfloor \kappa n \rfloor}(\hat{\mathbf{w}}^n_t)\| \ge \epsilon\|\sigma_{\lfloor \kappa n \rfloor}(\hat{\mathbf{w}}^n_t)\|\big] \\ 
\le & \mathbb{P}\big[ \|\sigma_{\lfloor \kappa n \rfloor}(\hat{\mathbf{w}}^n_t)\| \ge \epsilon\|\sigma_{\lfloor \kappa n \rfloor}(\hat{\mathbf{w}}^n_t)\|, E^c \big] + \mathbb{P}\big[\|\sigma_{\lfloor \kappa n \rfloor}(\hat{\mathbf{w}}^n_t)\| \ge \epsilon\|\sigma_{\lfloor \kappa n \rfloor}(\hat{\mathbf{w}}^n_t)\|, E \big]\\
\le & \mathbb{P}\left[\sup_{i\le n}\|\hat\theta^{i,n}_{\lfloor t/\eta \rfloor} - \theta^{i,\infty}_t\| > \frac{C}{\epsilon\sqrt{n}}\right] + \mathbb{P}\left[ \|\sigma_{\lfloor \kappa n \rfloor}(\mathbf{w}^\star_t)\| \ge \epsilon \|\mathbf{w}^\star_t\|-C\sqrt{n}(1+\kappa/\epsilon)\right] \\
\le &\epsilon/3 + \mathbb{P}\left[ \|\sigma_{\lfloor \kappa n \rfloor}(\mathbf{w}^\star_t)\| \ge \epsilon \|\mathbf{w}^*_t\|-C\sqrt{n}(1+\kappa/\epsilon)\right]
\end{aligned}
\end{equation} 

Moreover, there exists $N'>0$ such that for all $n\ge N'$,
\begin{equation}\label{eq:compression_2}
\begin{aligned}
&\mathbb{P}\left[ \|\sigma_{\lfloor \kappa n \rfloor}(\mathbf{w}^\star_t)\| \ge \epsilon \|\mathbf{w}^*_t\|-C\sqrt{n}(1+\kappa/\epsilon)\right]\\
\le & \mathbb{P}\left[ \|\mathbf{w}^\star_t\| \le 2C\sqrt{n}(1+\kappa/\epsilon)\right] 
+ \mathbb{P}\left[\|\sigma_{\lfloor \kappa n \rfloor}(\mathbf{w}^\star_t) \|\ge \frac{\epsilon}{2}\|\mathbf{w}^\star_t\| \right] \\
= & \mathbb{P}\left[ \frac{1}{n}\|\mathbf{w}^\star_t\|^2 \le 4C^2(1+\kappa/\epsilon)^2\right] 
+ \mathbb{P}\left[\|\sigma_{\lfloor \kappa n \rfloor}(\mathbf{w}^\star_t) \|\ge \frac{\epsilon}{2}\|\mathbf{w}^\star_t\| \right] \\
\le & \epsilon/3 + \mathbb{P}\left[\|\sigma_{\lfloor \kappa n \rfloor}(\mathbf{w}^\star_t) \|\ge \frac{\epsilon}{2}\|\mathbf{w}^\star_t\| \right],
\end{aligned}
\end{equation}
where the last inequality follows from Lemma \ref{lm:avr_inf}. By the independence of the $n$ coordinates of the vetor $\mathbf{w}^\star_t$, Theorem \ref{thm:heavy-tail} and [GCD12, Proposition 1, Part 2], there exists $N''>0$, for all $n\ge N''$,
\begin{equation}\label{eq:compression_3}
\mathbb{P}\left[\|\sigma_{\lfloor \kappa n \rfloor}(\mathbf{w}^\star_t) \|\ge \frac{\epsilon}{2}\|\mathbf{w}^\star_t\| \right] \le \epsilon/3.
\end{equation}
Finally, combining~\eqref{eq:Frob_norm}, \eqref{eq:compression_1}, \eqref{eq:compression_2} and \eqref{eq:compression_3} terminates the proof.
\end{proof}

\section{Experimental Details and Additional Results}
\label{sec:app:exp}
%{\color{red}prepare the code on github}

\subsection{Software and Hardware Requirements}
The experiments have been implemented in Python, using the deep learning framework PyTorch. Experiments were run on the server of an educational institution, using NVIDIA 1080 and 1080 Ti GPUs. The experiments published in the main paper and the Appendix amounted to an estimated GPU time of 1200 hours in total. Pruning and analysis is estimated to have taken an additional 40 GPU hours. We provide instructions to replicate and explore our results in the source code: \url{https://github.com/mbarsbey/implicit-compressibility}.

\subsection{Datasets}
The ECG5000 dataset \citep{ECG5000} consists of 5000 20-hour long electrocardiograms interpolated by sequences of length 140 to discriminate between
normal and abnormal heart beats of a patient that has severe congestive heart failure. After random shuffling, we use 500 sequences for the training phase and 4500 sequences for the test phase. The MNIST database \citep{mnist2010} of black and white handwritten digits consists of a training set of 60,000 examples and a test set of 10,000 examples of dimensions 28 x 28. CIFAR10 and CIFAR100 are two other image classification datasets \citep{cifar102009}, including 32 x 32 x 3 color images of objects or animals, making up 10 and 100 classes, respectively. We use the default split of 50,000 training and 10,000 test examples.

\subsection{Models and Training Hyperparameters}
The models used in the experiments are fully connected networks (FCN) and convolutional neural networks (CNN). All models include ReLU activations, and do not include any bias nodes nor any advanced layer structures such as batch normalization or residual connections. Due to the number of parameters being low compared to other layers, last linear layers of the models are not added noise during training, and are not included in pruning or computation of pruning ratios during evaluation. As described in the paper, we use FCNs with 1 or 4 hidden layers in different experiments. The CNN used in the experiments is a modified version of VGG11 \citep{simonyanVeryDeepConvolutional2015}, and has the following structure 
$$
128, M, 256, M, 512, 512, M, 1024, 1024, M, 1024, 1024, M,
$$
where the numbers refer to convolutional layer widths with 3 x 3 filters, followed by ReLU activation functions, and $M$'s refer to 2 x 2 max pooling operations.

All models in all experiments are trained until 95\% training accuracy, after which the training is concluded. As described in the main paper, no adaptive optimizers has been used in any of the experiments. For each experiment, $\sigma$ values have been selected to be as large as possible, without incurring dramatic performance loss and/or divergence on the un/pruned trained network. Learning rates, batch sizes, and $\sigma$ values have been provided in the Table \ref{tbl:hyperparameters}. Note that for CIFAR10, FCN experiments in Tables \ref{tbl:cifar10_1_5k}, we increase $\sigma$ to $0.0003$ to further illustrate the effects of noise added training on parameter compressibility. We also highlight that the batch sizes for CIFAR10, CNN and CIFAR100 experiments have been selected to be considerably smaller to other MNIST and CIFAR10 experiments, due to the former being more memory intensive. Given the very limited additional computational overhead of our method, our approach can easily be combined with standard hyperparameter selection methods. 

\begin{table}[h!]
\centering
\resizebox{\textwidth}{!}{
\begin{tabular}{ |c|c  c c c c |}
 \hline
Experiment & LR & B & $\sigma \ (\alpha=1.75)$ & $\sigma \ (\alpha=1.8)$ & $\sigma \ (\alpha=1.9)$ \\
\hline\hline
ECG5000, FCN, $n=2$K, Type-I & 0.0001 & 500 & 0.50 & 0.30 & 0.25 \\
\hline
ECG5000, FCN, $n=10$K, Type-I & 0.0001 & 500 & 1.25 & 1.00 & 0.75 \\
\hline
ECG5000, FCN, $n=10$K, Type-II & 0.0001 & 500 & 1.50& 1.75 & 2.25 \\
\hline
ECG5000, FCN, $n=10$K, Type-III & 0.0001 & 500 &3.00 &3.00  &2.00  \\
\hline
MNIST, FCN, $n=5$K, Type-I & 0.25 & 5000 & 0.001 & 0.00125 &  0.0011\\
\hline
CIFAR10, FCN, Type-I & 0.10 & 5000 & 0.0001 & 0.0001  & 0.0001 \\
\hline
% CIFAR10, FCN, Type-I (Rob.) & 0.10 & 5000 & 0.003 &  0.003 &  0.003 \\
% \hline
CIFAR10, CNN, Type-I & 0.01 & 100 &0.000075 & 0.0001 & 0.000075 \\
\hline
CIFAR100, FCN, Type-I & 0.01 & 100 &0.00005 & 0.000075 & 0.00009 \\
\hline

\end{tabular}
}
\caption{Hyperparameters for the experiments presented in the main paper and the appendices, including learning rate (LR), batch size (B), and chosen noise scales ($\sigma$) for various noise tail indices ($\alpha$).}
\label{tbl:hyperparameters}
\end{table}

\section{Implications on Federated Learning}

The federated learning (FL) setting~\citep{mcmahan2017,fedlearn2017} is one in which there are a number of devices or clients, say $n$; all equipped with the same neural network model and each holding an independent own dataset. Every client learns an individual (or local) model from its own dataset, e.g., via Stochastic Gradient Descent (SGD). The individual models are aggregated by a \textit{parameter server} (PS) into a global model and then sent back to the devices, possibly over multiple rounds of communication between them. The rationale is that the individually learned models are
refined progressively by taking into account the data held by other devices; and, at the end the training process, all relevant features of all devices' datasets are captured by the final aggregated model.

The results of this paper are useful towards a better understanding of the \textit{compressibility} of the models learned by the various clients in this FL setting. Specifically,  viewing each neuron of the hidden layer of the setup of this paper as if it were a distinct client, the results that we  establish \textit{suggest} that if the local models are learned via heavy-tailed SGD this would enable a better compressibility of them. This is particularly useful for resource-constrained applications of FL, such as in telecommunication networks where bandwidth is scarce and latency is important.

% \section{Illustration of Noise Types}

% \begin{figure}[h!]
% \centering
%      \includegraphics[width=0.4\textwidth]{alpha_stable_noises.png}
%      \caption{Samples of $1000$ planar $\alpha$-stable random vectors. Type-I noises lie in one direction; Type-II noises are i.i.d. distributed in the direction of x and y axis; Type-III noises has the same distribution in all directions.}
%      \label{fig:enter-label}
% \end{figure}
%%%%%%%%%%%%%%%%%%%%%%%%%%%%%%%%%%%%%%%%%%%%%%%%%%%%%%%%%%%%%%%%%%%%%%%%%%%%%%%
%%%%%%%%%%%%%%%%%%%%%%%%%%%%%%%%%%%%%%%%%%%%%%%%%%%%%%%%%%%%%%%%%%%%%%%%%%%%%%%

\end{document}